\documentclass[verbose=true,letterpaper]{article}
\PassOptionsToPackage{table}{xcolor}

\IfFileExists{neurips_2026.sty}{%
  \usepackage[preprint,nonatbib]{neurips_2026}%
}{\usepackage{review_layout}}            % preprint (e.g., arXiv)
\usepackage[utf8]{inputenc}
\usepackage[T1]{fontenc}
\usepackage{microtype}
\usepackage[toc,page]{appendix}
\usepackage[table]{xcolor}
\usepackage[
    hidelinks,
]{hyperref}

\usepackage[backend=biber, maxnames=1]{biblatex}
\usepackage{titlesec}
\usepackage{array}
\usepackage{amsmath}

\usepackage{amsfonts}
\usepackage{amssymb}
\usepackage{amsxtra}
\usepackage{amsthm}
\usepackage{mathrsfs}
\usepackage{mathtools}
\allowdisplaybreaks[1]
\usepackage{bm}
\usepackage{cancel}
\usepackage{yhmath}
\usepackage{nicefrac}

\usepackage{graphicx}
\usepackage{enumitem}
\usepackage{etoolbox}
\usepackage{needspace}
\usepackage{booktabs}
\usepackage{subcaption}
\usepackage{float}
\usepackage{tikz}
\usepackage{siunitx}
\usepackage{algorithm}
\usepackage{algpseudocode}

\usepackage{url}

\DeclareMathOperator*{\argmin}{arg\,min}

\newtheorem{assumption}{Assumption}

\newtheorem{corollary}{Corollary}
\newtheorem{definition}{Definition}
\newtheorem{remark}{Remark}
\newtheorem{theorem}{Theorem}

\usepackage{booktabs} % For professional table formatting
\usepackage{multirow} % For multi-row cells
\usepackage{adjustbox} % To adjust table width
\usepackage{array} % For custom column alignment
\usepackage{arydshln} % For dashed horizontal lines

\theoremstyle{plain}

\theoremstyle{definition}

\theoremstyle{remark}

\title{Deep Learning for Random Objects in Metric
Spaces: Beyond Manifold Assumptions}

\author{%
  Marcos Matabuena\thanks{Corresponding author.} \\
  Mohamed bin Zayed University of Artificial Intelligence
 \\
  \texttt{Marcos.Matabuena@mbzuai.ac.ae} \\
  \And
  Xiaokai Luo \thanks{MM and XL share first authorship.}
 \\
  University of Notre Dame \\
  \texttt{xluo4@nd.edu} \And
    Oscar Hernan Madrid Padilla \\
  Department of Statistics at University of California, Los Angeles \\
  \texttt{oscar.madrid@stat.ucla.edu} \\
}

\begin{document}
\maketitle

\smallskip
\begin{abstract}

Deep learning (DL) methods have substantially improved predictive performance across a wide range of downstream prediction tasks, enabling new discoveries in multiple domains. However, principled DL methods for non-Euclidean random objects remain relatively scarce, despite the growing importance of modern biomedical outcomes taking the form of probability distributions, random graphs, functions, and other  random objects in metric spaces. Motivated by a new generation of digital and clinical biomarkers,  we propose a deep learning framework for regression with random objects in metric spaces. Our approach operates directly on the metric structure of the response space, avoiding restrictive manifold or Euclidean embedding assumptions imposed by many existing methods. At the same time, it retains modern statistical guarantees  under the stated statistical assumptions. We establish statistical consistency and convergence rates for the proposed estimator, and characterize conditions under which it can mitigate the curse of dimensionality for hierarchically compositional regression functions. Extensive simulation studies demonstrate competitive predictive performance, with gains over several competing approaches in multiple Euclidean and non-Euclidean configurations. Finally, we demonstrate the practical utility of the proposed framework in two digital health case studies with distribution-valued outcomes in Wasserstein spaces, using high-resolution accelerometry and continuous glucose monitoring data. Embedding-based deep-learning methods may be less suitable when no informative low-dimensional representation exists or when the learned embedding does not adequately preserve the intrinsic geometry of the response space. 
\end{abstract}
\noindent%
\vspace{-1.2em}
%{\it Keywords:}
%m
%\vfill
%\spacingset{1.8} % DON'T change the spacing!
\section{Introduction}

Technological advances in biomedicine are generating increasingly dense and nearly continuous measurements of human physiology, behavior, and environment \cite{crainiceanu2024functional}. Digital health data  offer new opportunities to characterize human health at an unprecedented resolution and to develop novel personalized medicine strategies \cite{matabuena2026randomeffects}. At the same time, they pose new analytical challenges that are not naturally addressed by conventional multivariate statistical methods. In many modern biomedical applications, the outcome of interest is not adequately represented by a scalar or a finite-dimensional vector, but instead takes the form of a random object belonging to a metric space \((\mathcal{Y},d)\). Examples include probability distributions, networks, trees, and other random objects for which the usual operations of addition and scalar multiplication may not be available \cite{petersen2019frechet}. Random objects in metric spaces offer a natural representation for emerging digital biomarkers because they preserve aspects of physiological and behavioral variation that are often lost when high-resolution data are reduced to a few conventional summary measures.

A prominent example arises in continuous glucose monitoring (CGM). Rather than characterizing an individual's glucose profile exclusively through summary measures such as mean glucose, glucose variability, or time spent within prespecified glycemic ranges, the entire collection of glucose measurements can be represented through a \emph{glucodensity}, that is, a probability distribution describing the relative amount of time an individual spends across the full range of observed glucose values \cite{matabuena2021glucodensities, matabuena2025glucodensity}. Similar distributional representations have been proposed for other high-resolution digital-health measurements, including accelerometer-derived physical activity, where the distribution of activity intensities provides an individualized representation without relying on prespecified physical activity thresholds \cite{matabuena2023accelerometer}.

Figure~\ref{fig:initial-participant-heatmaps} illustrates an extension of this perspective using CGM data from three individuals. Rather than considering only the overall distribution of glucose values during the monitoring period, glucose measurements are characterized according to the time of day at which they are observed, yielding a temporal distributional representation of the daily glucose profile. This formulation is particularly appealing for physiological data collected under free-living conditions, where individuals differ in meal timing, sleep patterns, physical activity, and other daily behaviors. Consequently, exact pointwise alignment of the observed trajectories may be neither biologically meaningful nor scientifically appropriate. The distributional representation instead provides a unified statistical framework for defining biomarkers that retain relevant information about human health while enabling meaningful comparisons across individuals \cite{ghosal2026survival}.

\begin{figure}[ht!]
    \centering
    \includegraphics[width=\textwidth]
    {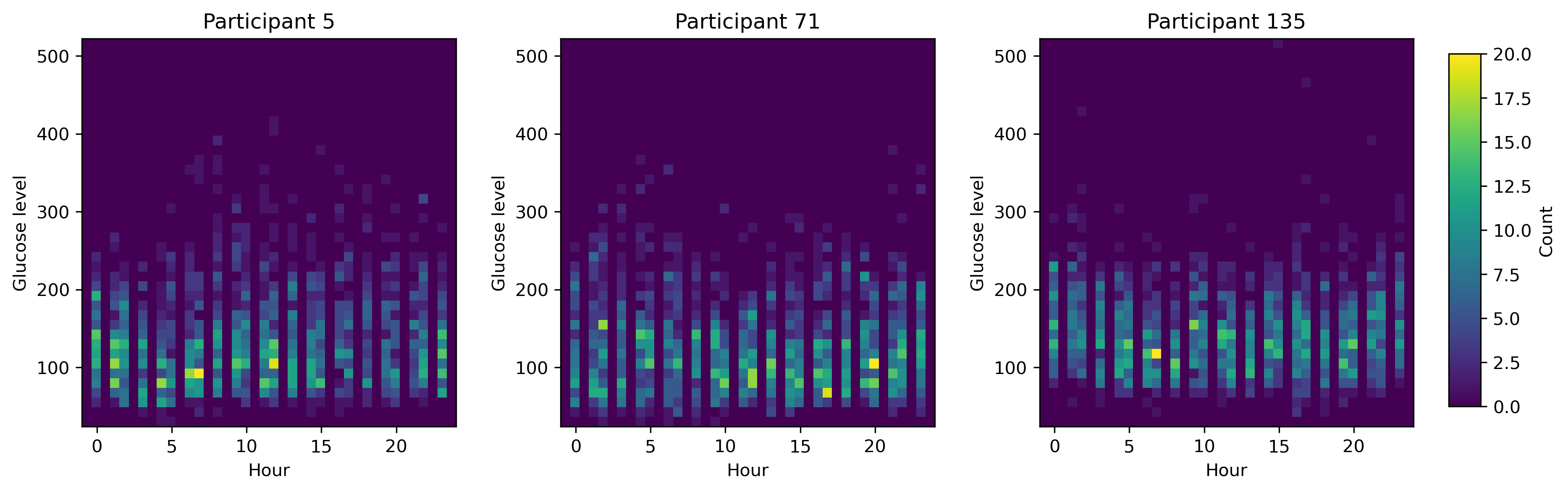}
    \caption{Bivariate temporal distributions derived from continuous
    glucose monitoring data for three participants. The horizontal axis
    represents time of day, and the vertical axis represents glucose
    concentration.}
    \label{fig:initial-participant-heatmaps}
\end{figure}

The increasing availability of distribution-valued and other non-Euclidean data has motivated the development of statistical methods for random objects in metric spaces \cite{frechet1948elements}. A fundamental framework is Fréchet regression, which extends the notion of conditional expectation by defining the regression target through a conditional Fréchet mean \cite{petersen2019frechet}. This perspective has since led to a broad range of methods for regression, dimension reduction, variable selection, longitudinal and multilevel data, and other statistical problems involving complex random objects. Nevertheless, most existing regression procedures are either designed for relatively low-dimensional predictors or rely on specific geometric or structural assumptions about the response space. Developing flexible regression methods that simultaneously accommodate general statistical responses and high-dimensional predictors therefore remains an important challenge in contemporary applications.

Deep learning (DL) \cite{lecun2015deep} is a natural candidate for addressing this problem, but DL methods for metric-space-valued responses remain comparatively underdeveloped. Several complementary deep-learning strategies have recently emerged. One approach uses an embedding-based construction. Deep Fréchet Regression first represents the relevant response structure in a low-dimensional Euclidean space, for example using ISOMAP, estimates the relationship between the predictors and the resulting coordinates using deep neural networks, and then maps predictions back to the original response space through local Fréchet regression \cite{iao2025deep}. This approach connects manifold learning and deep regression; its analysis relies on a learned low-dimensional representation followed by local Fr\'echet regression, so its performance depends on how well that representation captures the regression structure. Another recent strategy avoids explicitly embedding the responses into a Euclidean space and instead constructs predictions directly in the original metric space. The E2M framework of \cite{zhou2026e2m} employs a neural network to learn input-dependent weights on the probability simplex, with predictions obtained as weighted Fréchet means of the observed responses. A related approach is the deep Fréchet neural network (DFNN) framework of \cite{kim2026dfnn}, which combines learned representations of the predictors with a global Fréchet regression output layer and establishes universal approximation and generalization guarantees for metric-space-valued responses. Both frameworks provide end-to-end learning architectures that account for the geometry of the response space without requiring an explicit Euclidean embedding of the responses. However, their predictions are defined through weighted Fréchet optimization problems, making their implementation dependent on the geometric properties of the response space and the computational tractability of the corresponding optimization problems.

A separate line of work extends global Fr\'echet regression to nonlinear settings using operator-based methods. In particular, nonlinear global Fréchet regression based on weak conditional expectation uses reproducing-kernel Hilbert space embeddings and covariance operators to capture nonlinear relationships involving random objects \cite{bhattacharjee2025nonlinear}. This approach provides a flexible and theoretically wellgrounded
framework, although its implementation requires estimating and regularizing covariance
operators and related spectral quantities, which can be computationally and statistically demanding
in high-dimensional or non-Euclidean settings.

Despite these advances, developing statistically flexible and computationally tractable deep learning methods for metric-space-valued responses remains challenging. Existing approaches address the lack of linear structure in the response space through geometric embeddings, weighted Fréchet regression, or operator-based constructions. These strategies provide important theoretical and methodological foundations, but also motivate alternative formulations that can directly leverage the extensive methodological and theoretical developments in scalar-valued deep learning. In particular, modern deep learning theory has established favorable approximation and convergence guarantees for regression functions exhibiting hierarchical, compositional, or intrinsically low-dimensional structure, potentially mitigating the curse of dimensionality. This raises a fundamental question: can metric-space regression be reduced to conventional scalar-valued learning problems while transferring these statistical advantages to general non-Euclidean responses?

This challenge is especially interesting in view of recent developments in scalar nonparametric regression. Classical nonparametric methods are highly flexible, but their convergence rates typically deteriorate rapidly as the dimension of the predictors increases \cite{gyorfi2002distribution,tsybakov2009introduction}. Deep neural networks, however, can behave quite differently when the regression function has additional structure. A growing body of approximation and statistical theory shows that neural networks can achieve favorable rates for functions with hierarchical, compositional, sparse, or intrinsically low-dimensional structure \cite{yarotsky2017error,bauer2019deep,schmidthieber2020nonparametric}. In these settings, deep architectures can substantially alleviate the classical curse of dimensionality. This naturally raises the question of whether similar advantages can be retained when the response is a general random object rather than a scalar or Euclidean vector.

Motivated by these developments, we introduce a deep learning
framework for regression with responses taking values in metric
spaces. Our approach builds on the anchor-based regression
framework with random effects introduced by Matabuena and
C\'amara \cite{matabuena2026randomeffects}, extending its
underlying principles to modern deep learning. The proposed
framework is designed to retain the statistical flexibility
of deep neural networks while imposing minimal structure on
the response space. Rather than embedding the responses into
a low-dimensional manifold or Euclidean representation, we
reduce the original regression problem to a collection of
standard scalar-valued learning tasks and subsequently
reconstruct the target directly in the original metric space.
This formulation allows each component to be estimated using
conventional neural-network methods and, importantly, enables
the individual models to be trained independently and in
parallel.

Our contributions are threefold. First, we introduce an anchor-based deep learning framework for regression with metric-space-valued responses. The proposed approach transforms the original regression problem into a collection of scalar-valued learning tasks, thereby allowing the use of standard neural-network architectures without requiring an explicit Euclidean embedding of the responses. These learning tasks can be trained independently and in parallel, after which predictions are reconstructed in the original metric space. Second, we establish statistical consistency and convergence-rate guarantees for the proposed estimator under explicit assumptions on the response space, the regression structure, and anchor coverage. In particular, we investigate how hierarchical compositional structure in the underlying scalar regression functions influences the convergence behavior of the resulting metric-space estimator. Our analysis connects modern deep learning theory with anchor-based metric regression and identifies conditions under which favorable high-dimensional learning rates can be retained for non-Euclidean responses. Third, we investigate the finite-sample performance of the proposed framework through extensive simulation experiments and biomedical applications involving distribution-valued responses. We compare the proposed estimator with existing approaches, including Deep Fréchet Regression \cite{iao2025deep}, to evaluate its predictive accuracy and computational behavior across different regression settings.

\subsection{Notation}
For $n\in\mathbb{N}$, let $[n]=\{1,\ldots,n\}$. For $g:[0,1]^p\to\mathbb{R}$, define $\|g\|_\infty=\sup_{x\in[0,1]^p}|g(x)|.$ For $A>0$, let $g^A$ denote the clipping of $g$ at level $A$, $g^A(x)=\max\{-A,\min\{g(x),A\}\}.$

Given the design points $x_1,\ldots,x_n\in[0,1]^p$, define the normalized empirical norm
$\|g\|_n=\left(n^{-1}\sum_{i=1}^n g(x_i)^2\right)^{1/2}$,
omitting the dependence on the design points unless needed. The predictor dimension $p$ is fixed in the asymptotic statements. Network depths, widths, and hierarchy levels are positive integers; derivative orders may be zero.

For positive sequences $\{a_n\}$ and $\{b_n\}$, we write $a_n=O(b_n)$, or equivalently $a_n\lesssim b_n$, if there exist constants $C,N>0$ such that $a_n\leq Cb_n$ for all $n\geq N$. We write $a_n\asymp b_n$ if both $a_n=O(b_n)$ and $b_n=O(a_n)$. For random variables $\{X_n\}$, we write $X_n=o_{\mathbb{P}}(a_n)$ if $\frac{X_n}{a_n}\xrightarrow{\mathbb{P}}0,$ \noindent that is, for every $\varepsilon>0$, $\lim_{n\to\infty}\mathbb{P}\left(|X_n|>\varepsilon a_n\right)=0$.

\subsection{Outline}
This paper is organized as follows. Section~\ref{sec:met} introduces the methodology, and Section~\ref{theory} develops the statistical learning theory. Section~\ref{sec:simulation} reports the simulation study, while Section~\ref{sec:res} presents two digital-health case studies with multivariate distribution-valued outcomes. Section~\ref{sec:dis} concludes with limitations and future directions. The Supplementary Material contains the technical proofs and supporting results.

\section{Methodology}\label{sec:met}  
\subsection{Preliminaries}
 Let $(\mathcal{Y},d)$ be a separable metric space, and let $Y$ be a random object taking values in $\mathcal{Y}$. We assume that a regular conditional law of $Y$ given $X=x$ is fixed and that the risks below are jointly measurable and finite. Throughout the theoretical analysis, all minimizers are assumed to exist and to be found exactly. This includes the population targets, the network estimators, and the best-approximation functions. When a minimizer is not unique, a fixed measurable selection is used. Suppose that, together with $Y$, we observe a predictor vector $X \in \mathcal{X} \subset \mathbb{R}^p$. Throughout, we assume that the data consist of an i.i.d.\ sample $\mathcal{D}_n = \{(X_i,Y_i)\}_{i=1}^n \subset \mathcal{X}\times\mathcal{Y}$ from the joint distribution of $(X,Y)$. A fundamental notion for random objects in metric spaces is the Fr\'echet mean \cite{frechet1948elements}, which generalizes the classical expectation to non-Euclidean settings. The Fr\'echet mean is defined as
\begin{equation}
\mathbb{E}_{\oplus}[Y]
:=
\operatorname*{arg\,min}_{y\in\mathcal{Y}}
\mathbb{E}\bigl[d^2(y,Y)\bigr].
\end{equation}
Uniqueness requires additional conditions; for example, it holds in Hadamard spaces under a finite second-moment assumption \cite{sturm2003}. In the presence of covariates, the main statistical object of interest is the conditional Fr\'echet mean \cite{petersen2019frechet}, which generalizes the classical notion of conditional mean. It is defined by
\begin{equation}
m(x)
=
\mathbb{E}_{\oplus}(Y\mid X=x)
:=
\operatorname*{arg\,min}_{y\in\mathcal{Y}}
\mathbb{E}\bigl[d^2(y,Y)\mid X=x\bigr],
\qquad x\in\mathcal{X}.
\end{equation}
Here, the expectation is taken with respect to the conditional distribution of $Y$ given $X=x$. When $\mathcal{Y}=\mathbb{R}$ and $d$ is the Euclidean distance, the conditional Fr\'echet mean reduces to the classical conditional expectation. Regression with metric-space-valued responses can therefore be formulated as the problem of estimating the conditional Fr\'echet mean function $m(\cdot)$.

We next present two representative regression problems on metric spaces, both arising naturally in modern biomedical applications and playing an important role in contemporary statistical and machine learning methodologies.

\paragraph{Example 1 (Wasserstein distributions).}
Let \(\mathcal{X}=\mathbb{R}^{p}\), and let \(\mathcal{Y}=\mathcal{W}_{2}\) be the space of Borel probability measures on \(\mathbb{R}^{d}\) with finite second moments, equipped with the \(2\)-Wasserstein metric \(W_2\), whose square is
\[
W_2^2(\mu,\nu)
=
\inf_{\pi\in\Pi(\mu,\nu)}
\int_{\mathbb{R}^d\times\mathbb{R}^d}
\|u-v\|^2\,d\pi(u,v),
\]
\noindent where \(\Pi(\mu,\nu)\) is the set of couplings of \(\mu\) and \(\nu\). This metric compares distributions according to the least costly way of transporting mass, yielding a geometry-aware discrepancy that differs from classical divergence-based measures \cite{villani2003,panaretos2020invitation}.

\noindent In digital health, a common physiological representation is the distribution of a time series generated by a continuous process of interest such as glucose, heart rate, physical activity, or sleep over time. For subject \(i=1,\dots,n\), let \(\mathbf{Z}_i(t)\in\mathbb{R}^d\) be a digital-health time series observed at \(t_{i1}<\cdots<t_{im_i}\), e.g., continuous glucose, heart rate, or physical activity. The distribution-valued response is often the empirical marginal law
\[
\mu_i=\frac{1}{m_i}\sum_{j=1}^{m_i}\delta_{\mathbf{Z}_i(t_{ij})},
\]
\noindent or, to incorporate dynamics \cite{matabuena2025glucodensity}, the empirical joint law of the process and its discrete derivative,
\[
\nu_i=\frac{1}{m_i-1}\sum_{j=1}^{m_i-1}
\delta_{\left(\mathbf{Z}_i(t_{ij}),\,
\frac{\mathbf{Z}_i(t_{i,j+1})-\mathbf{Z}_i(t_{ij})}{t_{i,j+1}-t_{ij}}\right)}.
\]
\noindent With covariates \(X_i\in\mathbb{R}^{p}\), one models \(\mathbb{E}[\mu_i\mid X_i=x]\) or \(\mathbb{E}[\nu_i\mid X_i=x]\) in the Wasserstein space \(\mathcal{W}_2\); see \cite{matabuena2025predicting,matabuena2021glucodensities,matabuena2025glucodensity}. Thus the outcome is the full subject-specific distribution rather than a scalar summary, such as the mean biomarker level over the observation period.

\paragraph{Example 2 (Symmetric positive-definite matrices).}
Let \(\mathcal{X}=\mathbb{R}^{p}\) and \(\mathcal{Y}=\mathrm{Sym}_+^l\), the space of \(l\times l\) symmetric positive-definite matrices, which includes covariance and correlation matrices. Several natural metrics are available:
\begin{itemize}
\item Frobenius: \(d_F(A,B)=\|A-B\|_F\).
\item Log-Cholesky \cite{lin2019}: if \(A=L_A L_A^\top\), where \(L_A\) is lower triangular with positive diagonal, then
\[
d_{\mathrm{LC}}^2(A,B)=\|\log L_A-\log L_B\|_F^2.
\]
\item Power Euclidean \cite{dryden2009}: 
\[
d_\alpha(A,B)=\frac{1}{\alpha}\|A^\alpha-B^\alpha\|_F,\qquad \alpha>0,
\]
with the log-Euclidean limit \(\|\log A-\log B\|_F\) as \(\alpha\to0\).
\end{itemize}
See also \cite{pennec2006,bhatia2019} for affine-invariant and Bures--Wasserstein metrics. Such matrices arise naturally as covariance or correlation summaries of multivariate biomedical signals, for example, fMRI functional connectivity across brain regions, EEG channel covariance, or diffusion tensor imaging \cite{zalesky2012use}.
 This paper develops a deep-learning (DL) estimator of the conditional Fr\'echet mean that operates directly on metric-space-valued responses without requiring a prespecified Euclidean or manifold embedding.

\subsection{Deep ReLU neural networks}
We define the rectified linear unit (ReLU) activation function by $\tau(x) = \max\{0,x\}, \qquad x \in \mathbb{R}.$

Following \cite{kohler2021rate,padilla2024confidence}, we restrict attention to
multilayer feedforward ReLU networks in which all hidden layers have the same
width. Specifically, we let $\mathcal{F}(L,\nu)$ denote the class of dense ReLU
networks with $L$ hidden layers and $\nu$ hidden units in each layer.

Throughout, the term ``dense'' refers to the absence of sparsity restrictions
on the network parameters; equivalently, we consider fully connected network
architectures. An explicit definition of the functions belonging to
$\mathcal{F}(L,\nu)$ is provided in Appendix~\ref{sec:neural_network_class}.

\subsection{Estimator}

Before defining the estimator, define the conditional Fr\'echet risk function $M:\mathcal{X}\times\mathcal{Y}\to\mathbb{R}$ by

\[
M(x,y) := \mathbb{E}\bigl[d^2(Y,y)\mid X=x\bigr].
\]
\noindent The target of interest is the conditional Fr\'echet mean function $m : \mathcal{X} \to \mathcal{Y}$, which satisfies
\[
m(x) := \underset{y\in\mathcal{Y}}{\arg\min}\, M(x,y).
\]
\noindent Recalling that  $\mathcal{D}_{n}=\{(X_i,Y_i)\}_{i=1}^n$ are independent copies of $(X,Y)$, for each $i=1,\dots,n$, define the random function $M_i:\mathcal{X}\to\mathbb{R}$ by
\begin{equation}
\label{eqn:Mi}
M_i(x) := M(x,Y_i) = \mathbb{E}\bigl[d^2(Y,Y_i)\mid Y_i,  X=x\bigr].
\end{equation}
\noindent Thus, for each fixed $i$, the function $M_i$ describes the conditional expected squared distance between $Y$ and the observed response $Y_i$.

 An estimator of $M_i$ can be obtained by regressing $\{d^2(Y_j,Y_i):j\ne i\}$ on $\{X_j:j\ne i\}$. The observation indexed by $i$ is excluded so that, conditional on the anchor $Y_i$, the remaining regression pairs are independent. In \eqref{eqn:Mi}, $(X,Y)$ denotes a fresh copy independent of the training sample. Based on these estimators, we then construct an estimator of the conditional Fr\'echet mean function $m(\cdot)$. The procedure is defined as follows.

\begin{enumerate}
    \item {Estimation of the conditional risk functions.}

Let $I\subset\{1,\ldots,n\}$ be nonempty and deterministic. For each $i\in I$, define the exact empirical minimizer
\[
\widehat f_i\in\argmin_{f\in\mathcal F(L_i,\nu_i)}
\frac{1}{n-1}\sum_{j\in[n]\setminus\{i\}}
 \bigl(d^2(Y_j,Y_i)-f(X_j)\bigr)^2.
\]
The minimum is taken over the entire dense network class: no sparsity or weight constraints are imposed. Set
\[
\widehat M_i:=\widehat f_i^{A_n},\qquad A_n>0.
\]
Thus the clipping is applied after exact minimization.

    \item {Selection of the estimated Fr\'echet mean.}
    
    Having obtained $\{\widehat{M}_j\}_{j\in I}$, we predict a new covariate value $x$ by selecting the observed response whose estimated conditional risk is smallest. Specifically, we define
    \[
    \widehat{m}(x) := Y_i,
    \]
\noindent    where $i$ is chosen so that
    \[
 i = \min\!\left\{j\in I:\widehat M_j(x)=\min_{k\in I}\widehat M_k(x)\right\}.
    \]
\end{enumerate}

%\subsection{Computational Details}

 %We implemented the proposed algorithm in PyTorch. In the practical execution of the algorithm, we fix $I = \{1,\ldots,n\}$. All univariate regressions were performed using the same neural network architecture: a feedforward neural network with four hidden layers of width $20$. Each model was trained for $1000$ epochs using a learning rate of $0.05$ and a batch size of $40$. The hyperparameters were fixed across all experiments to ensure consistency and  fair comparison across different settings examined.

\section{Theory}
\label{theory}

We first establish risk bounds for arbitrary ReLU neural-network classes, expressed in terms of approximation error, and then specialize these bounds to the conditional risk functions $M_i$ under the hierarchical-composition assumptions of \cite{kohler2021rate}.

\subsection{General Theory for Estimation with ReLU Neural Networks}
Throughout we  assume that $\mathcal{X} = [0,1]^p$ for some $p\in \mathbb{N}$. We now state our main conditions on the metric $d$.

\begin{assumption}
    \label{as1}
The metric space $\mathcal{Y}$ is bounded: $d(a,b)\le C_n$ for all $a,b\in\mathcal Y$. The notation $C_n$ is retained to allow an $n$-dependent bound.
\end{assumption}

\begin{remark}
Assumption~\ref{as1} is satisfied whenever the metric space
$(\mathcal Y,d)$ has finite diameter,
\[
\operatorname{diam}(\mathcal Y)
:=
\sup_{a,b\in\mathcal Y}d(a,b)
<\infty.
\]
In this case, one may take
\[
C_n=\operatorname{diam}(\mathcal Y)
\]
for every $n$. The following are standard examples.

\begin{enumerate}

\item {Bounded Euclidean responses.}
Let
\[
\mathcal Y
\subseteq
\left\{
y\in\mathbb R^q:\|y\|_2\leq R
\right\},
\qquad
d(a,b):=\|a-b\|_2,
\]
for some $R<\infty$. Then, for every $a,b\in\mathcal Y$,
\[
d(a,b)
=
\|a-b\|_2
\leq
\|a\|_2+\|b\|_2
\leq
2R.
\]
Thus Assumption~\ref{as1} holds with $C_n=2R$. More generally, every
bounded subset of a finite-dimensional normed space satisfies
Assumption~\ref{as1} under the metric induced by the norm.

\item {Bounded functional or Hilbert-space-valued responses.}
Let $\mathbb H$ be a Hilbert space and let
\[
\mathcal Y
:=
\left\{
f\in\mathbb H:\|f\|_{\mathbb H}\leq R
\right\},
\qquad
d(f,g):=\|f-g\|_{\mathbb H}.
\]
Then
\[
d(f,g)
\leq
\|f\|_{\mathbb H}+\|g\|_{\mathbb H}
\leq
2R,
\]
so Assumption~\ref{as1} holds with $C_n=2R$. In particular, this
includes bounded subsets of $L^2(\mathcal T)$ equipped with the
$L^2$ distance.

\item {Probability measures under a Wasserstein metric.}
Let $(\mathcal S,\rho)$ be a bounded metric space with
\[
D:=\sup_{u,v\in\mathcal S}\rho(u,v)<\infty,
\]
and let $\mathcal Y=\mathcal P(\mathcal S)$ be a collection of
probability measures on $\mathcal S$. For $r\geq1$, equip
$\mathcal Y$ with the $r$-Wasserstein metric
\[
W_r(\mu,\nu)
:=
\left[
\inf_{\pi\in\Pi(\mu,\nu)}
\int_{\mathcal S\times\mathcal S}
\rho(u,v)^r\,\pi(du,dv)
\right]^{1/r},
\]
where $\Pi(\mu,\nu)$ denotes the set of couplings of $\mu$ and $\nu$.
Since $\rho(u,v)\leq D$ for all $u,v\in\mathcal S$,
\[
W_r(\mu,\nu)^r
\leq
D^r,
\]
and therefore
\[
W_r(\mu,\nu)\leq D.
\]
Thus Assumption~\ref{as1} holds with $C_n=D$. For example, probability
distributions supported on a bounded interval or a bounded subset of
$\mathbb R^q$ satisfy this condition under the Wasserstein metric.

\item {Probability distributions under the Hellinger metric.}
Let $\mathcal Y$ be a collection of probability measures dominated by
a common measure $\lambda$, and define
\[
d_{\mathrm H}(P,Q)
:=
\left[
\int
\left(
\sqrt{\frac{dP}{d\lambda}}
-
\sqrt{\frac{dQ}{d\lambda}}
\right)^2
\,d\lambda
\right]^{1/2}.
\]
Since
\[
d_{\mathrm H}(P,Q)^2
=
2-
2\int
\sqrt{\frac{dP}{d\lambda}
      \frac{dQ}{d\lambda}}
\,d\lambda
\leq2,
\]
we have
\[
d_{\mathrm H}(P,Q)\leq\sqrt{2}.
\]
Hence Assumption~\ref{as1} holds with $C_n=\sqrt{2}$, even when the
underlying sample space is unbounded.

\item {Finite and compact metric response spaces.}
Every finite metric space is bounded. Indeed, if
\[
\mathcal Y=\{y_1,\ldots,y_m\},
\]
then
\[
C_n
=
\max_{1\leq j,k\leq m}d(y_j,y_k)
<\infty.
\]
More generally, every compact metric space is bounded. For example,
the unit sphere
\[
\mathbb S^{q-1}
=
\{y\in\mathbb R^q:\|y\|_2=1\}
\]
has diameter $2$ under the Euclidean chordal metric and diameter
$\pi$ under the geodesic metric.

\end{enumerate}

In all of these examples, $C_n$ can be chosen to be a fixed constant
that does not depend on $n$. The notation $C_n$ also permits settings
in which the allowed response set grows with $n$; for example, if
$\mathcal Y$ is restricted to a ball of radius $R_n$, one may take
$C_n=2R_n$.
\end{remark}

\begin{remark}
    In digital-health applications, Assumption 1 is particularly natural, as physiological measurements are typically acquired within finite measurement ranges and may be restricted to bounded domains during preprocessing. Moreover, sensor acquisition and preprocessing often involve smoothing or filtering operations that improve the regularity of the observed signals. When the resulting distribution-valued responses are supported on a common bounded domain, their Wasserstein distances are uniformly bounded, and Assumption 1 is satisfied.
\end{remark}

For a measurable function $g:[0,1]^p\to\mathcal Y$, define
\[
\mathcal R(g):=\mathbb E[d^2(Y,g(X))]=\int M(x,g(x))\,dP_X(x).
\]
When $g$ is learned from $\mathcal D_n$, $\mathcal R(g)$ means the conditional prediction risk
$\mathbb E[d^2(Y,g(X))\mid\mathcal D_n]$ for a fresh independent pair $(X,Y)$.
Thus $\mathcal R(\widehat m)$ is random through the training sample, and
$\mathcal R(m)\le\mathcal R(g)$ holds for every realized measurable $g$.

\noindent We are now ready to state our main result which applies to arbitrary ReLU neural network function classes.

\begin{theorem}
\label{theo1}
Suppose Assumption~\ref{as1} holds. Let $L,\nu:\mathcal Y\to\mathbb N$ be deterministic measurable functions, allowed to depend on $n$, with $L(y)\nu(y)\ge2$. Set $L_i=L(Y_i)$ and $\nu_i=\nu(Y_i)$. For every $y\in\mathcal Y$, let
\[
\widetilde M(\cdot,y)\in\argmin_{f\in\mathcal F(L(y),\nu(y))}
\|f-M(\cdot,y)\|_\infty
\]
be an exact best approximant, and choose a positive sequence $\phi_n$ such that
\[
\sup_{y\in\mathcal Y}\|\widetilde M(\cdot,y)-M(\cdot,y)\|_\infty^2\le\phi_n.
\]
Choose the clipping level $A_n$ deterministically. Suppose
\begin{equation}
\label{eqn:an}
A_n\ge8\max\left\{1,\ \sup_y\|M(\cdot,y)\|_\infty+\mathcal U_n,
\ 8\sup_y\|M(\cdot,y)\|_\infty,\ 8\sqrt{\phi_n}\right\},
\end{equation}
where $\mathcal U_n=C_n^2$. Then
\[
\begin{split}
\mathcal R(\widehat m)\le\mathcal R(m)+r_n+c_1\Bigg[
&\sqrt{\frac{|I|\max\{\mathcal U_n^2,A_n^2\}
\max_{i\in I}\{L_i^2\nu_i^2\log(L_i\nu_i)\}\log(A_nn)}{n}}
+\sqrt{|I|\phi_n}\Bigg]
\end{split}
\]
with probability approaching one, for a constant $c_1>0$ independent of $n$, provided that $r_n\downarrow0$,
\[
\frac{\min_{i\in I}\{L_i^2\nu_i^2\log(L_i\nu_i)\}\log(2A_n^2)}{\log n}
\xrightarrow{\mathbb P}\infty,
\]
and
\begin{equation}
\label{cond2}
C_n\sup_{y\in\mathcal Y}\min_{i\in I}d(Y_i,y)=o_{\mathbb P}(r_n).
\end{equation}
 The probabilities are over $\mathcal D_n$. Constants may depend on the fixed predictor dimension $p$ and the stated uniform bounds.
\end{theorem}

\begin{remark}
\label{rem2}
When $A_n$ and $\mathcal U_n$ grow at most polylogarithmically, Theorem~\ref{theo1} gives risk consistency if the coverage condition holds and
\[
\frac{|I|\max\{\mathcal U_n^2,A_n^2\}
\max_{i\in I}\{L_i^2\nu_i^2\log(L_i\nu_i)\}\log(A_nn)}{n}=o_{\mathbb P}(1),
\qquad |I|\phi_n=o(1).
\]
The coverage rate depends on both the geometry of $\mathcal Y$ and the distribution of the anchors. 
\end{remark}

\begin{corollary}
    \label{cor1}
  Suppose that Assumption \ref{as1} holds. In addition, assume that the parameters $\{L_i\}_{i \in [n]}$ and $\{\nu_i\}_{i \in [n]}$ can be chosen such that the estimates $\widehat{M}_i$ satisfy  
   \begin{equation}
\label{eqn:cond3}
\underset{i\in I}{\max} \| \widehat{M}_i - M_i\|_{\infty} \,=\, O_{\mathbb{P}}(\alpha_n)
   \end{equation}
\noindent for a positive sequence $\alpha_n$ that converges to zero. Then, provided that \eqref{cond2} holds,
        \[
       \mathcal R(\widehat m)-\mathcal R(m)=O_{\mathbb P}(r_n+\alpha_n).
    \]
\end{corollary}

\subsection{Theory under Hierarchical Composition Models}

We now present our main result under explicit assumptions on the functions $M(\cdot,y)$ for $y \in \mathcal{Y}$. In particular, we impose conditions intended to mitigate the curse of dimensionality that arises when $X \in [0,1]^p$ is high-dimensional. To this end, we follow \cite{kohler2021rate}, who studied nonparametric regression with ReLU networks under a hierarchical composition class $\mathcal{H}(l,\mathcal{P})$ assumption. This class consists of functions formed through successive compositions of smooth component functions, each of which may depend on only a lower-dimensional subset of the $p$ input variables.  The relevant definitions are given in Appendix~\ref{sec:hierchical}.
%settings where these networks can effectively overcome the curse of dimensionality. The relevant definitions are given in Appendix~\ref{sec:hierchical}.

%We now state explicit assumptions on the functions $M(\cdot,y)$ for $y \in \mathcal{Y}$. In particular, we impose conditions designed to mitigate the curse of dimensionality that arises because $X \in [0,1]^p$ may be high-dimensional. To this end, we follow \cite{kohler2021rate}, who studied nonparametric regression with ReLU networks in settings where such networks can effectively overcome the curse of dimensionality. Before presenting our assumptions, we first recall some standard notation from \cite{kohler2021rate}.

We are now ready to state the main result of this subsection.

\begin{theorem}\label{theo2}

Suppose Assumption~\ref{as1} holds. For every $y\in\mathcal Y$, assume
$M(\cdot,y)\in\mathcal H(l(y),\mathcal P(y))$, with a chosen representation whose finite, nonempty set of component pairs is $\mathcal P(y)\subset[1,\infty)\times\mathbb N$.
Each component $g$ is $(t_g,C)$-smooth in the sense of Definition~\ref{definition: p,c smoothness}, where $t_g=q_g+s_g$, $q_g\ge0$, $s_g\in(0,1]$, and $(t_g,K_g)\in\mathcal P(y)$. The smoothness constant $C$, the bounds $c_2$ on all derivatives of orders at most $q_g$ (including order zero), and the Lipschitz constant $C_{\mathrm{Lip}}\ge1$ are uniform in $g$ and $y$. Assume
\[
\sup_y l(y)<\infty,\qquad
\sup_y\max_{(t,K)\in\mathcal P(y)}\max\{t,K\}<\infty.
\]
Let
\[
\phi_n:=\sup_{y\in\mathcal Y}\max_{(t,K)\in\mathcal P(y)}n^{-2t/(2t+K)}.
\]

Let $L(y)$ and $\nu(y)$ be as in \eqref{eqn:v2_choice1} or \eqref{eqn:v2_choice2}, with sufficiently large constants uniform in $y$, and assume these architecture maps are measurable. With $A_n$ satisfying \eqref{eqn:an} and $\mathcal U_n$ as in Theorem~\ref{theo1},

\[
\mathcal R(\widehat m)\le\mathcal R(m)+r_n+
 c_1\sqrt{|I|\max\{\mathcal U_n^2,A_n^2\}\phi_n
             (\log n)^3\log(A_nn)}
\]

with probability approaching one, provided $r_n\downarrow0$ and the coverage condition \eqref{cond2} holds. The constants may depend on the fixed $p$, the uniform hierarchy-depth and smoothness bounds, $c_2$, and $C_{\mathrm{Lip}}$, but not on $n$ or $y$.
\end{theorem}

When $A_n=O(\operatorname{poly}(\log n))$,  condition \eqref{eqn:an}, which imposes a lower bound on $A_n$ in terms of $\mathcal U_n$, $\sup_y\|M(\cdot,y)\|_\infty$, and $\sqrt{\phi_n}$, also implies that $\mathcal U_n$ grows at most polylogarithmically. In this case, Theorem~\ref{theo2} yields a remainder of order $r_n+\sqrt{|I|\phi_n}$ up to logarithmic factors. In particular, since $\log(A_nn)\asymp\log n$, the square-root term carries an additional factor $(\log n)^2$, together with the factor $\max\{\mathcal U_n,A_n\}$. The rate exponent in $\phi_n$ depends on the component dimensions rather than on the ambient predictor dimension. The resulting bound is therefore a general upper bound under the stated metric-space coverage conditions, rather than a claim of optimality. The corresponding scalar neural-network rate is given in \cite{kohler2021rate}.

\begin{remark}[Number of anchors]
Theorem 2 suggests that, for large sample sizes, selecting
a subset of anchors with cardinality
$|I|=k_n\asymp n^\alpha$ may be preferable to using all
observations. In particular, if
$\phi_n\asymp n^{-2t/(2t+K)}$ and the clipping and diameter
terms remain bounded, the estimation-error term vanishes
whenever
\[
0<\alpha<\frac{2t}{2t+K}.
\]
Consequently, the excess-risk bound converges to zero
provided that the anchor-coverage condition and the
remaining assumptions of Theorem 2 are satisfied.
\end{remark}

Our convergence results complement and extend recent theoretical
developments for deep learning with metric-space-valued responses.
Iao, Zhou, and M\"uller \cite{iao2025deep} derive convergence
rates under a low-dimensional manifold representation of the
conditional Fr\'echet mean, accounting for errors arising
from estimating the manifold coordinates.
Zhou, Iao, and M\"uller \cite{zhou2026e2m} establish universal
approximation and convergence properties for the
entropy-regularized E2M training objective.
Kim, Chen, and Dubey \cite{kim2026dfnn} establish universal
approximation, uniform consistency, and generalization
guarantees for DFNN, including excess-risk and prediction-error
bounds under neural-network complexity controls. Our work makes a distinct theoretical and methodological
contribution by establishing a direct connection between
deep learning theory for hierarchically compositional functions
and regression with general metric-space-valued responses.
Unlike existing approaches, our framework reduces the
non-Euclidean regression problem to independently estimable
scalar conditional squared-distance regression tasks,
without requiring a latent manifold representation or an
end-to-end Fr\'echet optimization layer.
More importantly, building on Kohler and Langer
\cite{kohler2021rate}, we establish explicit excess-risk
bounds that exploit the hierarchical compositional structure
of these scalar regression functions.
The resulting statistical estimation rates are governed
by the smoothness and effective dimensions of the
component functions rather than directly by the ambient
predictor dimension.

This result identifies a theoretically important class
of high-dimensional non-Euclidean regression problems
for which deep learning can mitigate the classical
curse of dimensionality.
In contrast to general approximation and neural-network
complexity guarantees, our analysis makes explicit how
favorable compositional structure translates into
statistical learning guarantees for metric-space-valued
responses.
Crucially, the theory also characterizes the role of
the number and coverage of response anchors, revealing
the trade-off between scalar estimation accuracy and
approximation of the conditional Fr\'echet target.
Thus, our framework provides not only an alternative
computational strategy, but also a distinct theoretical
perspective on when and why deep learning can be
statistically effective for non-Euclidean outcomes.

\section{Simulation Study}
\label{sec:simulation}

We assess the finite-sample performance of the proposed method through simulation experiments. Our main goal is to evaluate how accurately it estimates the conditional Fr\'echet mean and how it compares with competing approaches for regression with metric-space-valued responses. The competing methods are deep Fr\'echet regression (DFR), global Fr\'echet regression (GFR), \(k\)-nearest-neighbor regression (KNN), and random forests (RF).

We consider two response types: (i) Euclidean-valued outcomes and (ii) probability distributions in the \(2\)-Wasserstein space. In all scenarios, the predictor vector is \(X=(X_1,\ldots,X_p)^\top\in\mathbb{R}^p\), with \(p\in\{5,10,15,20\}\). The training sample size is varied over \(n\in\{200,300,400,500\}\), and all fitted models are evaluated on an independent test set of size \(n_{\mathrm{test}}=2000\). The prediction error is defined as
\[
\frac{1}{n_{\mathrm{test}}}
\sum_{i\in\mathcal{D}_{\mathrm{test}}}
d^2\!\left(
Y_i,\widehat m(X_i)
\right),
\]
where \(Y_i\) is the test response corresponding to \(X_i\) and \(\widehat m(\cdot)\) is the fitted estimator. We use Euclidean distance for Euclidean responses and \(2\)-Wasserstein distance for distribution-valued responses. Each experiment is repeated over \(100\) independent Monte Carlo replications.

%The goal of this section is to illustrate the versatility of our deep-learning approach for predicting multivariate distributions. In the univariate case, \(\mathcal{W}_2(\mathbb{R})\) admits a convenient linear representation via quantile functions and embeds isometrically into \(L^2([0,1])\) \cite{villani2009optimal,petersen2019frechet}. Multivariate distributions, by contrast, do not admit an analogous Hilbert-space embedding that would allow linear methods to be applied directly \cite{villani2009optimal,ambrosio2008gradient}. As a result, competing deep-learning approaches for metric-space-valued outcomes that rely on ISOMAP \cite{tenenbaum2000global} may yield inaccurate estimates. ISOMAP estimates geodesic distances through shortest paths in a neighborhood graph. For multivariate distributions, however, the Wasserstein geometry is not locally Euclidean in the same sense as in the one-dimensional case, and reliable graph construction requires a large, well-spread sample---an assumption rarely satisfied in clinical datasets.

For all neural-network-based methods, we use three hidden layers with \(50\) ReLU units per layer. Models are trained for \(1000\) epochs with Adam, learning rate \(0.005\), and mini-batch size \(32\). The same architecture and optimization settings are used for all neural-network competitors.

\subsection{Euclidean Responses}

\subsubsection{Scenario 1: Mixed nonlinear regression}

We first consider a nonlinear regression setting in which both the number of active predictors and the complexity of the regression function increase with the total number of predictors \(p\). Predictors are generated independently as
\[
X_j\overset{\mathrm{iid}}{\sim}\operatorname{Unif}(-2,2),
\qquad
j=1,\ldots,p.
\]
The response is
\[
Y=m(X)+\varepsilon,
\qquad
\varepsilon\sim\mathcal{N}(0,0.1^2),
\]
with
\begin{align}
m(x)
={}&
\sin(x_1x_2)
+0.5\cos(x_3)
+0.25x_4^2
\notag\\
&+
\mathbb{I}(p\geq15)
\left(
x_{11}+x_{12}+x_{13}+x_{14}x_{15}
\right)
\notag\\
&+
\mathbb{I}(p\geq20)
\left(
x_{16}+x_{17}+x_{18}+x_{19}x_{20}
\right).
\label{eq:scenario1}
\end{align}

For \(p=5\) and \(p=10\), the response depends only on the first four predictors. For \(p=15\) and \(p=20\), we introduce additional active variables and interaction terms. This design allows us to study how the methods behave as both the ambient dimension and the intrinsic complexity of the problem increase.

Figure~\ref{fig:scenario1} shows that the proposed DL method improves substantially as the training sample size increases. For \(p=5\) and \(p=10\), KNN and RF are competitive when \(n\) is small, whereas DL becomes comparable or better as more training data become available. For \(p=15\) and \(p=20\), linear regression performs well in the smallest-sample settings, but DL eventually achieves the lowest prediction error. The relative performance of DL improves in the higher-dimensional configurations, where the regression function contains more active predictors and nonlinear interactions.

\begin{figure}[H]
\centering
\includegraphics[width=\textwidth]
{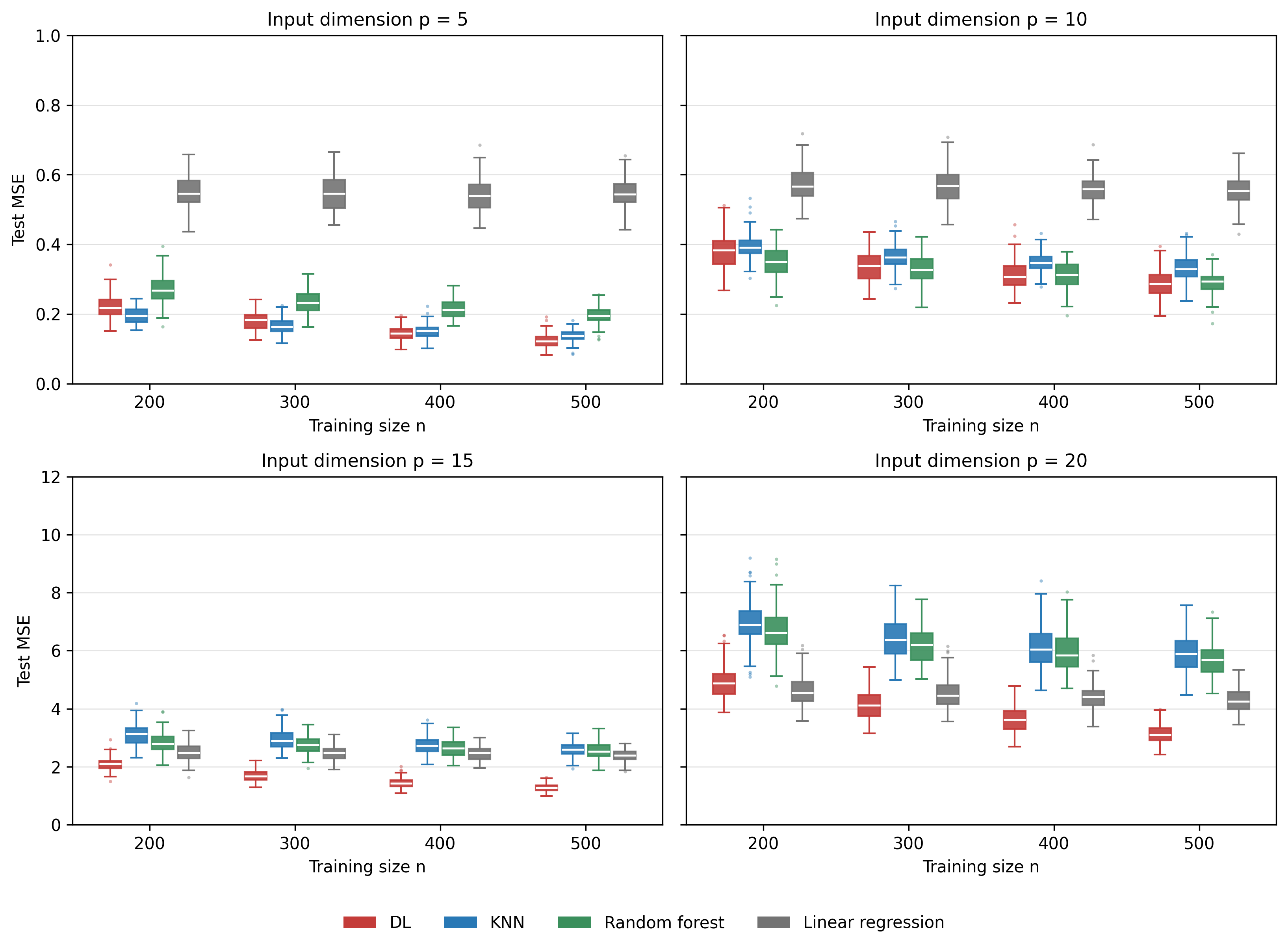}
\caption{Boxplots of prediction error for the competing methods in Scenario~1. Each panel corresponds to a different predictor dimension \(p\).}
\label{fig:scenario1}
\end{figure}

\subsubsection{Scenario 2: Fixed neural-network regression}

In the second scenario, the regression function is generated by a fixed
teacher neural network. The predictors are independently generated as
\[
X_j \overset{\mathrm{iid}}{\sim} \operatorname{Unif}(-2,2),
\qquad j=1,\ldots,p,
\]
and the response is
\[
Y=m(X)+\varepsilon,
\qquad
\varepsilon\sim\mathcal{N}(0,0.1^2).
\]

Let \(r=\min(p,8)\) and define
\[
\widetilde X=(X_1,\ldots,X_r)^\top.
\]
Thus, only the first \(r\) predictors affect the response. The teacher network
has three hidden layers with widths \(32\), \(24\), and \(16\):
\begin{align}
h_1(x)
&=
\tanh\!\left(W_1^\top\widetilde x+b_1\right),\\
h_2(x)
&=
\operatorname{ReLU}\!\left(W_2^\top h_1(x)+b_2\right),\\
h_3(x)
&=
\tanh\!\left(W_3^\top h_2(x)+b_3\right).
\end{align}

{
Here,
\[
W_1\in\mathbb{R}^{r\times32},\qquad
W_2\in\mathbb{R}^{32\times24},\qquad
W_3\in\mathbb{R}^{24\times16}.
\]

Writing
\[
(d_0,d_1,d_2,d_3)=(r,32,24,16),
\]
the hidden-layer weights are generated independently according to
\[
(W_\ell)_{ab}
\sim
\mathcal{N}\!\left(
0,\frac{2}{d_{\ell-1}+d_\ell}
\right),
\qquad \ell=1,2,3,
\]
and the corresponding bias components are independently generated from
\[
(b_\ell)_a\sim\mathcal{N}(0,0.2^2).
\]
The initial output weights and bias satisfy
\[
(\widetilde w_4)_a\sim\mathcal{N}\!\left(0,\frac{1}{16}\right),
\qquad
\widetilde b_4\sim\mathcal{N}(0,0.1^2),
\]
and define the uncalibrated signal
\[
\widetilde m(x)
=
\widetilde w_4^\top h_3(x)+\widetilde b_4.
\]

To standardize the signal strength, we generate \(20{,}000\) calibration
vectors independently from \(\operatorname{Unif}([-2,2]^r)\). Let
\(\overline m\) and \(s_m\) denote the empirical mean and standard deviation
of \(\widetilde m\) over this calibration sample. The final regression
function is
\[
m(x)
=
\frac{\widetilde m(x)-\overline m}{s_m}.
\]
}
All network parameters are generated once and held fixed across Monte Carlo replications, ensuring that every replication uses the same underlying regression function. The network output is standardized to have approximately zero mean and unit variance under the predictor distribution.

When $p>8$, only the first eight predictors affect the response; the remaining variables are noise predictors. Unlike Scenario~1, increasing $p$ therefore increases ambient dimension without increasing the intrinsic dimension of the regression function.

Figure~\ref{fig:scenario2} shows that DL performs particularly well for $p=5$ and becomes increasingly competitive for $p=10$, $15$, and $20$ as training sample size grows. Its prediction error decreases markedly with $n$, indicating that the proposed method recovers the nonlinear structure of the teacher network once enough training data are available.

\begin{figure}[H]
\centering
\includegraphics[width=\textwidth]
{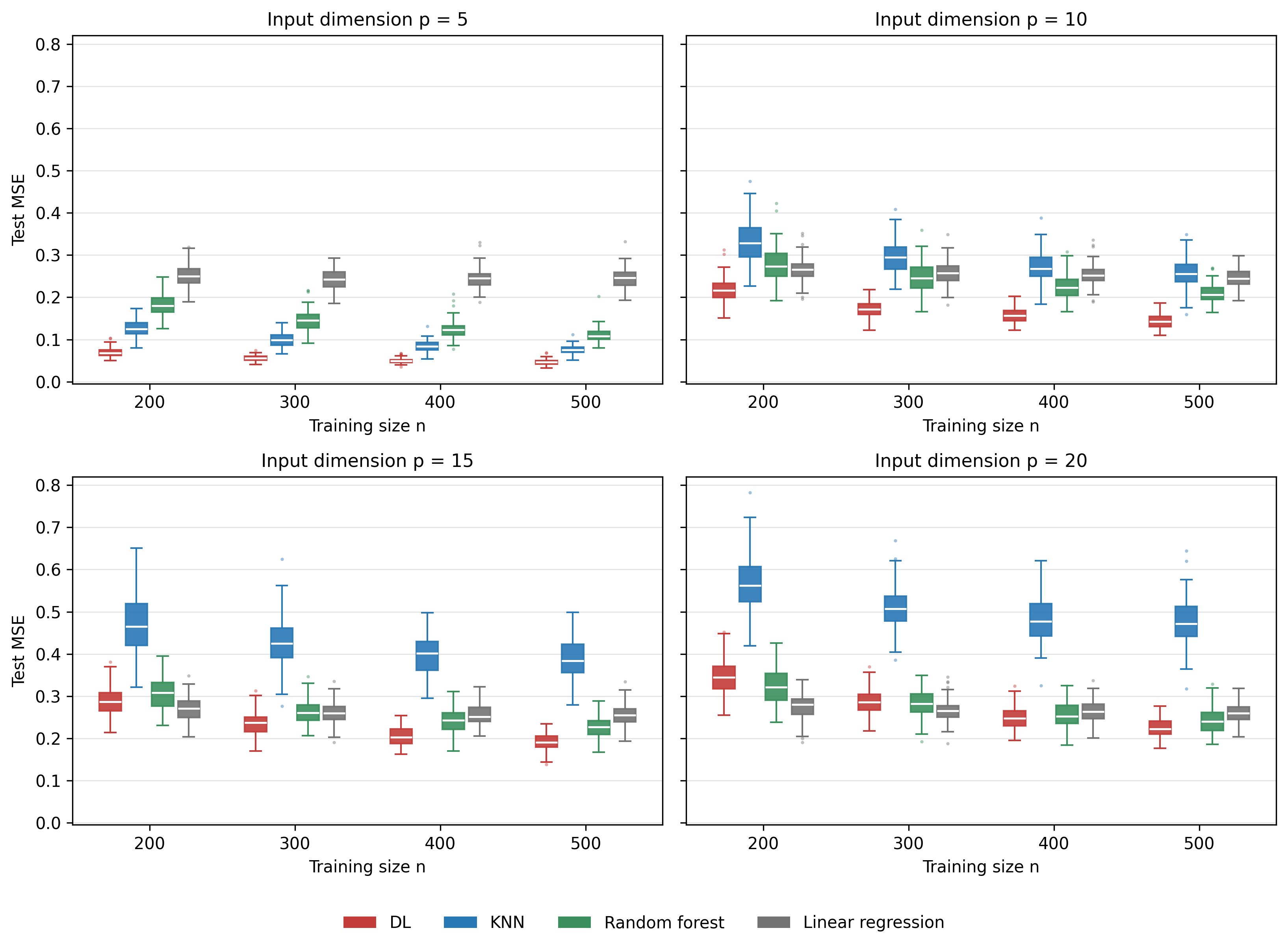}
\caption{Boxplots of prediction error for the competing methods in Scenario~2. Each panel corresponds to a different predictor dimension \(p\).}
\label{fig:scenario2}
\end{figure}

\subsection{Distribution-Valued Responses}

We next consider regression settings in which the response is a probability distribution belonging to the Wasserstein space \(\mathcal{W}_2(\mathbb{R}^q)\), where \(q\) denotes the dimension of the underlying Euclidean space on which the distributions are defined.

\subsubsection{Scenario 3: Covariate-dependent Gaussian distributions}

In the third simulation, the response is a probability distribution on \(\mathbb{R}^q\). For each training observation, we generate
\[
X_i \sim \operatorname{Uniform}([0,1]^p),
\]
where \(p\) is the predictor dimension. Conditional on \(X_i=x\), the response is a multivariate Gaussian
\[
P_i = \mathcal{N}\!\left(\mu(x), Q_\Sigma(x)\right),
\]
with
\[
\mu(x)=\left(\sum_{j=1}^{p} x_j,\sum_{j=1}^{p} x_j^2,\ldots,\sum_{j=1}^{p} x_j^q\right)^\top\in\mathbb{R}^q,
\]
\[
Q_\Sigma(x)=\Sigma(x)+U_i^2\sigma^2 I_q,
\qquad
\Sigma(x)=v(x)v(x)^\top+0.5I_q,
\qquad
v(x)=Wx.
\]

Here, \(W\in\mathbb{R}^{q\times p}\) has entries independently
generated from \(\mathcal{N}(0,0.5^2)\) and is held fixed
across all Monte Carlo replications.
Moreover, \(U_i\sim\operatorname{Uniform}(-1,1)\), and
\(\sigma>0\) controls the isotropic covariance perturbation.
We set \(\sigma=0.05\) throughout this simulation scenario.
The term \(0.5I_q\) ensures positive definiteness,
while the rank-one component allows the covariance to
vary smoothly with the predictors.

Each \(P_i\) is approximated by \(s_1=200\)
independent draws
\[
Y_{i1},\ldots,Y_{is_1}
\overset{\mathrm{iid}}{\sim}P_i,
\]
yielding the empirical distribution
\[
\widehat P_i
=\frac{1}{s_1}\sum_{\ell=1}^{s_1}\delta_{Y_{i\ell}},
\]
where \(\delta_y\) denotes the Dirac measure at \(y\).

For the test set, we generate
\(X\sim\operatorname{Uniform}([0,1]^p)\)
and define the Gaussian reference distribution
\[
P_X^{\mathrm{ref}}
=
\mathcal{N}\left(
\mu(X),\Sigma(X)+\frac{\sigma^2}{3}I_q
\right).
\]

\noindent This reference distribution is obtained by averaging
the conditional covariance matrices over the random
perturbation \(U\), using \(E[U^2]=1/3\).
It provides a deterministic benchmark for evaluating
the recovery of the underlying covariate-dependent
distributional structure. Although this reference
distribution does not necessarily coincide with the
conditional Fr\'echet mean under the 2-Wasserstein
metric, the distinction is expected to be negligible
in the present setting because the covariance
perturbation is small (\(\sigma=0.05\)).

Thus this scenario provides a nonlinear distributional regression benchmark in which both the location and dispersion of the response vary with the predictors. It assesses whether the competing methods recover smooth covariate-dependent structure in Wasserstein space while accounting for observation-level covariance variability.

Figure~\ref{fig:scenario3q1} reports the results for \(q=1\). DL achieves the lowest average \(W_2\) prediction error in 14 of the 20 configurations. It performs best across all sample sizes when \(p=5\), while for larger \(p\) its relative performance improves steadily with \(n\). DFR performs particularly well in higher-dimensional settings when the training sample size is small. KNN and GFR generally lie between these methods, while RF tends to produce the largest errors. Overall, DL benefits strongly from larger training samples and becomes particularly competitive once enough observations are available.

\subsubsection{Scenario 4: Product-interaction Gaussian distributions}

In the fourth simulation, the response is a probability distribution on $\mathbb{R}^q$, empirically represented by $s_1$ sampled points. For each training observation, the covariates are generated through a fixed feed-forward neural network. We first draw
\[
Z_i \sim \mathcal{N}(0,I_{d_z}),
\]
and then set
\[
X_i = c_X\{f_\theta(Z_i)+\varepsilon_i\},
\qquad
\varepsilon_i \sim \mathcal{N}(0,\sigma_X^2 I_p),
\]
where $f_\theta:\mathbb{R}^{d_z}\to\mathbb{R}^p$ is a feed-forward neural network with $\tanh$ hidden activations and a linear output layer. The parameters $\theta$ are generated once and held fixed within each Monte Carlo replication, so training and test covariates share the same nonlinear transformation.

The mean vector is
\[
\mu(X_i)
=
\bigl(\pi(X_i),\pi(X_i)^2,\ldots,\pi(X_i)^q\bigr)^\top,
\qquad
\pi(X_i)=\prod_{j=1}^{r} X_{ij},
\]
where $r\le p$ is the number of active covariates, with default $r=3$. Thus the mean depends on a multiplicative interaction among the first few predictors.

The covariance is
\[
\Sigma(X_i)=v(X_i)v(X_i)^\top+0.5\,I_q,
\qquad
v(X_i)=WX_i,
\]
with $W\in\mathbb{R}^{q\times p}$ fixed across the experiment. 

Within each Monte Carlo replication, \(W\) is generated once and shared by
the training and test observations. Its entries are independently generated as
\[
W_{kj}\sim\mathcal{N}(0,0.5^2).
\]

Both \(W\) and the covariate-generating neural network are regenerated across
Monte Carlo replications using replication-specific random seeds.

To introduce additional heterogeneity, define
\[
Q_{\Sigma}(X_i)=\Sigma(X_i)+U_i^2\sigma^2 I_q,
\qquad
U_i\sim\operatorname{Uniform}([-1,1]),
\]
where $\sigma>0$ controls the perturbation level. { We
set \(\sigma=0.05\) throughout this simulation scenario.} The response distribution is then
\[
P_i=\mathcal{N}\bigl(\mu(X_i),Q_{\Sigma}(X_i)\bigr),
\]
and $P_i$ is empirically represented by  { $s_1=200$ independent samples.} 

%For the test set, covariates are generated using the same neural-network parameters and covariance map \(W\) as the training set within that replication, but with independent latent variables, covariate noise, and covariance perturbations. Both the covariate-generating network and \(W\) are regenerated across Monte Carlo replications using replication-specific random seeds.

This scenario is a challenging distributional benchmark: the conditional distribution varies through both mean and covariance, and the mean contains a multiplicative interaction. It therefore tests whether each method captures interaction-driven distributional effects in Wasserstein space.

Figures~\ref{fig:scenario4q2} and~\ref{fig:scenario4q4} compare the methods for response dimensions $q=2$ and $q=4$. For $q=2$, KNN performs best when $p=5$ and $p=10$, while DL becomes increasingly competitive as $p$ grows and achieves the lowest mean error for all sample sizes when $p=20$. For $q=4$, KNN has the lowest mean error in 14 of the 20 configurations, while DL performs best in the remaining six and is particularly competitive for $p=20$. Overall, KNN is the most stable method, whereas the  performance of DL improves in the higher-dimensional settings considered. DFR and GFR generally exhibit intermediate performance, while RF produces the largest errors and variability.

\begin{figure}[H]
\centering
\includegraphics[width=\textwidth]
{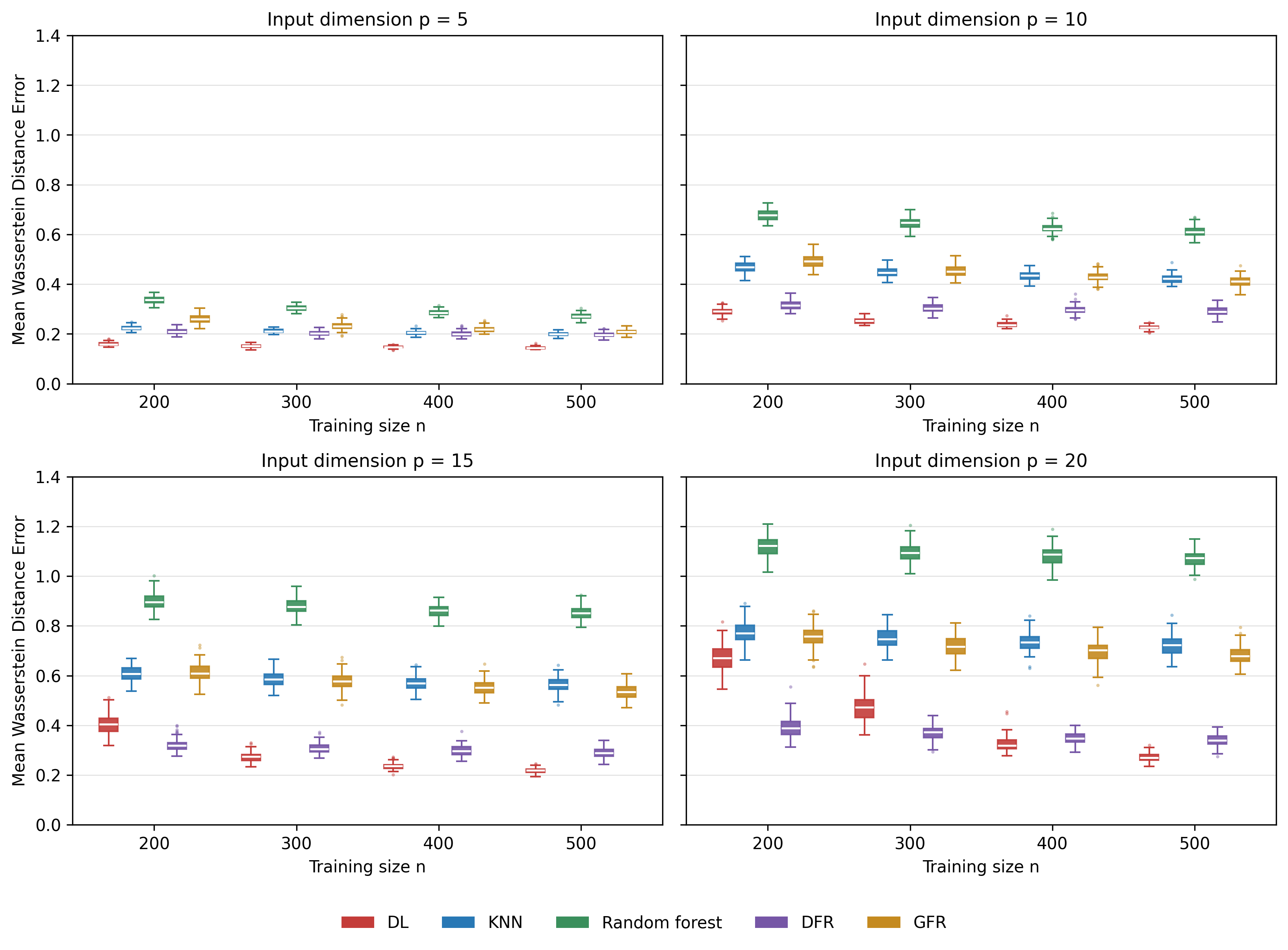}
\caption{Boxplots of prediction error for the competing methods in Scenario~3 with \(q=1\). Each panel corresponds to a different predictor dimension \(p\).}
\label{fig:scenario3q1}
\end{figure}

\begin{figure}[H]
\centering
\includegraphics[width=\textwidth]
{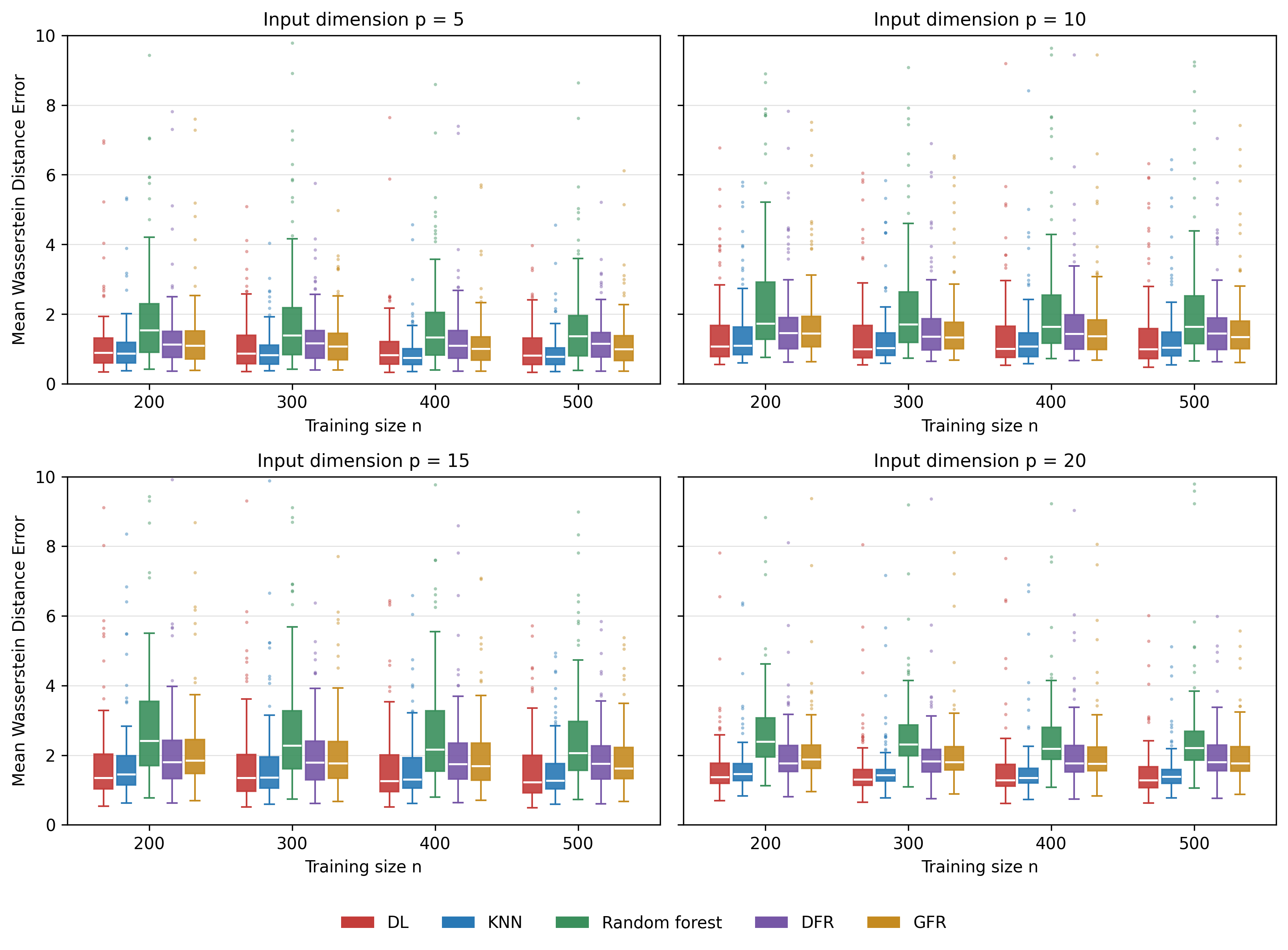}
\caption{Boxplots of prediction error for the competing methods in Scenario~4 with \(q=2\). Each panel corresponds to a different predictor dimension \(p\).}
\label{fig:scenario4q2}
\end{figure}

\begin{figure}[H]
\centering
\includegraphics[width=\textwidth]
{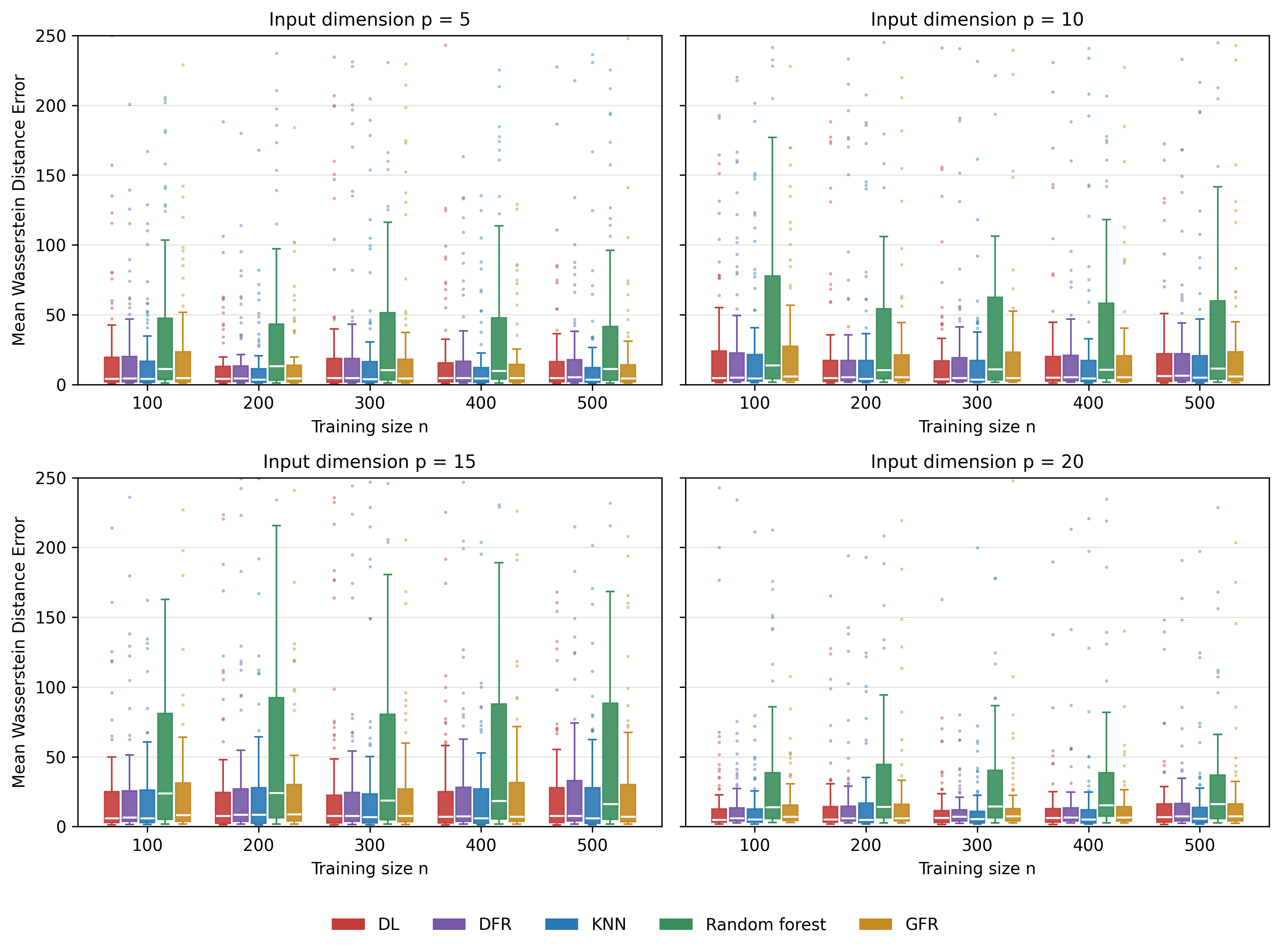}
\caption{Boxplots of prediction error for the competing methods in Scenario~4 with \(q=4\). Each panel corresponds to a different predictor dimension \(p\).}
\label{fig:scenario4q4}
\end{figure}

\section{Digital Health Case Studies Based on Distributional Data Analysis}
\label{sec:res}

\subsection{Motivation and Benchmark}
\label{sec:res-motivation}

Digital health technologies can now collect near-continuous physiological, environmental, and behavioral measurements of many clinically relevant biomarkers. The resulting data are, high-frequency times series that offer new opportunities for personalized interventions and improved population health. A natural way to represent the resulting biomarker time series is as random objects in a general metric space. In free-living conditions, for instance, time series of physiological signals are not directly comparable across subjects because meal times and other contextual factors differ. A natural representation of such a time series is therefore its probability distribution \cite{matabuena2021glucodensities}---for example, as an element of the Wasserstein space introduced earlier, which lacks a natural vector-space structure.

This section illustrates the versatility of our deep-learning approach for predicting multivariate distributions. In the univariate case, \(\mathcal{W}_2(\mathbb{R})\) admits a convenient linear representation via quantile functions and embeds isometrically into \(L^2([0,1])\) \cite{villani2009optimal,petersen2019frechet}. Multivariate distributions, by contrast, do not admit an analogous Hilbert-space embedding that would allow linear methods to be applied directly \cite{villani2009optimal,ambrosio2008gradient}. As a result, competing deep-learning approaches for metric-space-valued outcomes that rely on ISOMAP \cite{tenenbaum2000global} may yield inaccurate estimates. ISOMAP estimates geodesic distances through shortest paths in a neighborhood graph. For multivariate distributions, however, the Wasserstein geometry is not locally Euclidean in the same sense as in the one-dimensional case, and reliable graph construction requires a large, well-spread sample---an assumption rarely satisfied in clinical datasets.

\paragraph{Notation and outcome.}
For individual \(i\), let \(G_{ij}\in\mathbb{R}^s\), \(j=1,\dots,n_i\), denote the \(j\)th repeated measurement vector, where \(s\) is the biomarker dimension and \(n_i\) is the number of repeated measurements for that subject. We treat these measurements as samples from an individual-specific probability distribution \(F_i\) on \(\mathbb{R}^s\), and estimate the distributional response by the empirical distribution
\[
\widehat{F}_{i}
=
\frac{1}{n_i}\sum_{j=1}^{n_i}\delta_{G_{ij}}.
\]
The observed distributional response is \(Y_i := \widehat{F}_{i}\). In the glucose application below, each response is instead represented by a fixed number \(s_1=2000\) of samples; here \(s_1\) generically denotes the number of samples representing a response. To quantify dissimilarity between two responses \(Y_i\) and \(Y_j\), we use the \(2\)-Wasserstein distance \(\mathcal{W}_2(\cdot,\cdot)\) introduced in Example~1. The regression target is the conditional Fr\'echet mean
\[
m_\oplus(x)
=
\arg\min_{P\in\mathcal{Y}}
\mathbb{E}\bigl[
\mathcal{W}_2^2(Y,P)\mid X=x
\bigr],
\]
provided the minimizer exists and is unique. We estimate \(m_\oplus\) using the deep-learning method (DL) proposed in { Section~\ref{sec:met}}. Prediction error for a held-out subject is
\[
\widehat{\epsilon}_{i}
=
\mathcal{W}_2^2\bigl(Y_i,\widehat m_{\oplus}^{(-i)}(X_i)\bigr), \qquad i=1,\dots, n,
\]
where \(\widehat m_{\oplus}^{(-i)}\) is the estimator trained without subject \(i\).

\paragraph{Competing methods and evaluation.}
In both case studies we compare DL with random forests (RF) \cite{capitaine2024frechet}, global Fr\'echet regression (GFR) \cite{petersen2019frechet}, deep Fr\'echet regression (DFR) \cite{iao2025deep}, which uses an ISOMAP manifold step \cite{tenenbaum2000global}, and \(k\)-nearest neighbors (KNN) \cite{devroye1996probabilistic}. Performance is assessed using (i) the distribution of \(\widehat{\epsilon}_i\) on a test sample \(\mathcal{D}_{\mathrm{test}}\) of \(20\) individuals and (ii) a jackknife-based assessment over \(B=20\) random individuals, summarized by boxplots of the resulting error distributions. The same protocol is used in Sections~\ref{sec:physical} and~\ref{sec:glucose}.

\subsection{Results: 3D Physical Activity Distribution Data}
\label{sec:physical}

We analyze physical-activity data from NHANES 2011--2012. The sample contains \(n=632\) participants. For participant \(i\), \(G_{ij}\in\mathbb{R}^3\) denotes the \(j\)th repeated activity measurement along the three ActiGraph axes, and the response \(Y_i\) is the empirical distribution of these measurements, as defined in Section~\ref{sec:res-motivation}. The predictors \(X_i\) are age, sex, body mass index, grip strength, and four principal component scores derived from spirometry measurements, following the feature-extraction procedure described in prior work \cite{ghosal2026physical}. The goal is to predict a distributional response from individual-level covariates; error is measured by \(\mathcal{W}_2\) as in Section~\ref{sec:res-motivation}.

Figure~\ref{fig:sample_distribution} displays the estimated three-dimensional distributions for participants 115, 368, and 500.

\begin{figure}[ht!]
    \centering
     \includegraphics[width=0.95\textwidth]{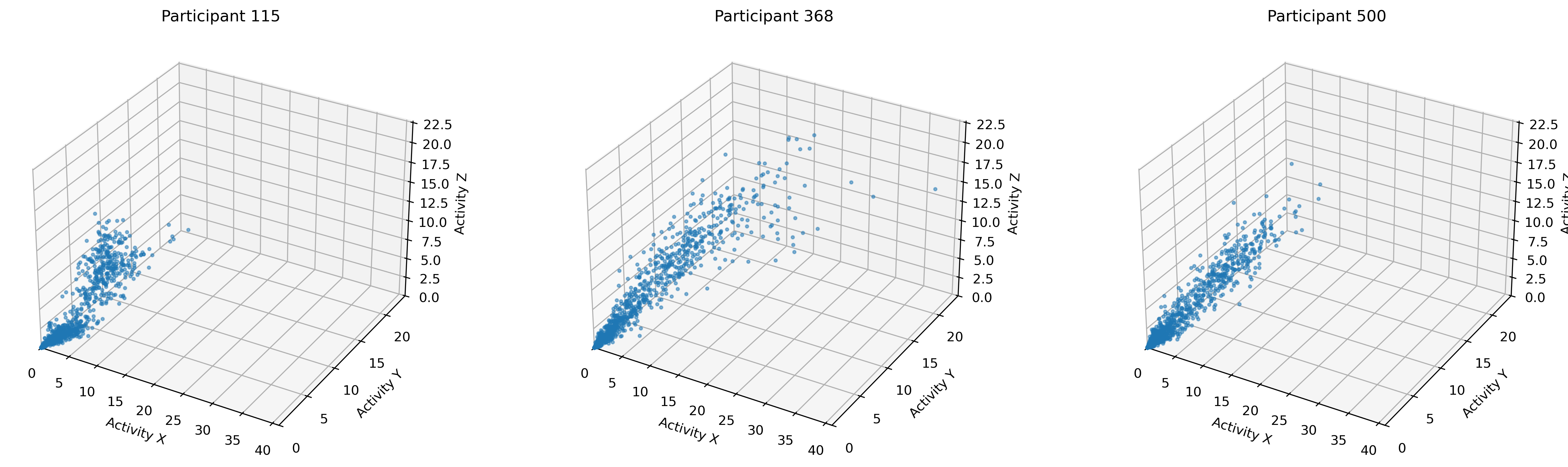}
    \caption{3D distributions for Participants 115, 368, and 500.}
    \label{fig:sample_distribution}
\end{figure}

Figure~\ref{fig:physical_lineplot} shows Wasserstein errors for the \(20\) held-out participants, and Figure~\ref{fig:physical_boxplot} summarizes the error distributions. DL has the smallest median error overall and performs competitively for most held-out participants, although it also produces a few larger errors. RF and KNN perform well, with error levels broadly comparable to DL. GFR and DFR yield substantially larger errors. In this application, DL yields lower errors than GFR and DFR while remaining broadly comparable to RF and KNN.

\begin{figure}[ht]
    \centering
    \begin{minipage}[t]{0.48\textwidth}
        \centering
        \includegraphics[width=\textwidth]{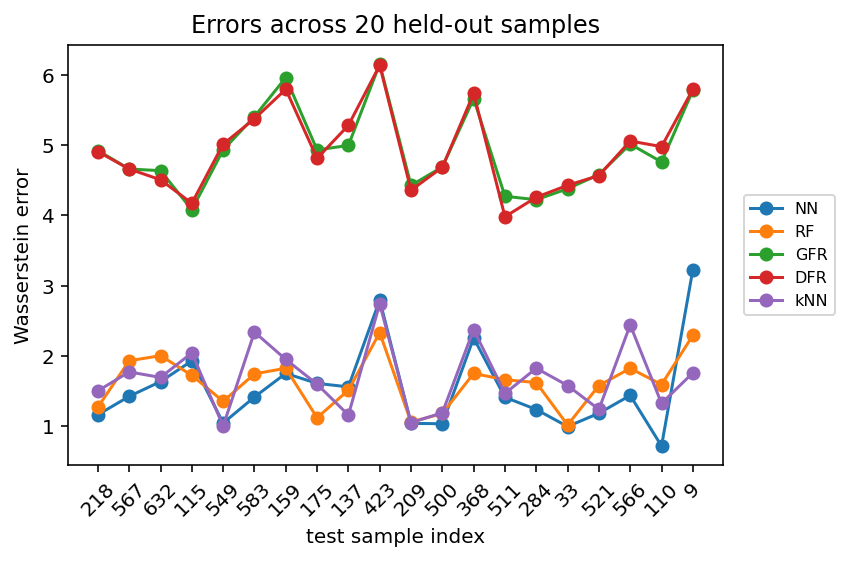}
        \caption{Physical activity error line plot.}
        \label{fig:physical_lineplot}
    \end{minipage}
    \hfill
    \begin{minipage}[t]{0.48\textwidth}
        \centering
        \includegraphics[width=\textwidth]{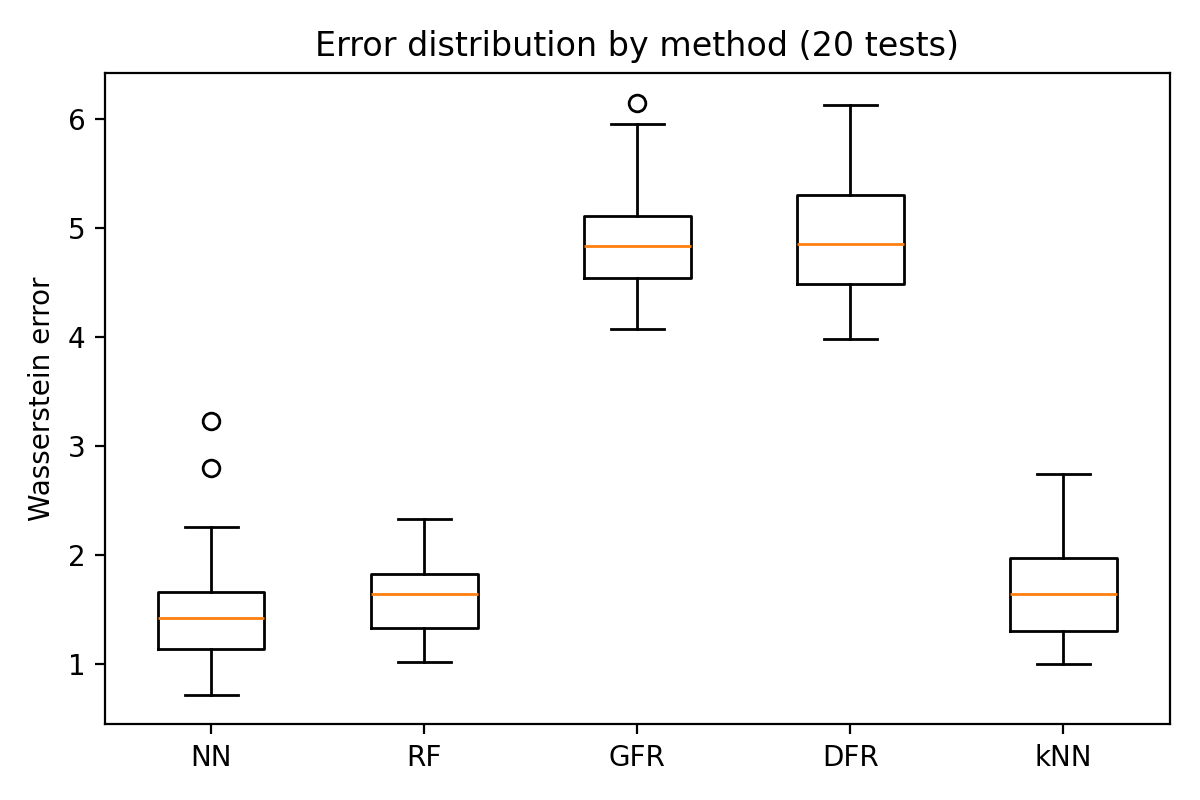}
        \caption{Physical activity error boxplot.}
        \label{fig:physical_boxplot}
    \end{minipage}
\end{figure}

\subsection{Glucose Distribution Data}
\label{sec:glucose}

We now evaluate DL on a glucose distribution dataset constructed from continuous glucose monitoring (CGM) measurements and clinical covariates. For each patient, we use four clinical predictors: \texttt{LabA1cResult}, \texttt{AgeAsOfRandDt}, \texttt{Weight}, and \texttt{DurDiabetes} used in our prior research \cite{matabuena2026multilevel}. After removing subjects with missing covariates, we obtain a processed dataset with \(399\) participants. For each  individual the response is the empirical distribution of the CGM measurements in the \((\text{hour},\text{glucose})\) plane; each subject is represented by \(s_1=2000\) samples.

We compare DL with RF, GFR, DFR, and KNN using a
leave-one-out validation procedure, as described in
Section~\ref{sec:res-motivation}. From the full sample
of 399 participants, we randomly select 20 individuals
for evaluation. For each selected participant, we
independently fit the competing models using the
remaining 398 participants and predict the
distribution-valued response of the held-out individual.
This procedure is repeated for all 20 selected
participants, yielding 20 out-of-sample predictions
per method. Prediction accuracy is then assessed
by comparing each predicted distribution with
the corresponding observed distribution..Performance is measured by the Wasserstein error \(\widehat{\epsilon}_i\) between the predicted and observed glucose distributions.

For DL, we first compute the pairwise  Wasserstein distance matrix \cite{bonneel2015sliced,peyre2019computational} among the training responses, then train a neural network mapping the clinical covariates to the corresponding distance profile. Given a test patient, the network predicts distances to the training subjects, and the final prediction selects the nearest training distribution. The network uses four hidden layers, each of width \(20\), with learning rate \(0.05\), batch size \(40\), and \(1000\) training epochs.

Figure~\ref{fig:initial-participant-heatmaps} shows representative empirical glucose distributions over \(24\) hours. Each panel displays the joint distribution of hour and glucose level, with vertical bands reflecting repeated measurements at each time point. Most glucose values lie between \(70\) and \(200\) mg/dL, but the empirical distributions also exhibit right-skewed tails with occasional hyperglycemic spikes above \(300\) mg/dL. There is clear heterogeneity across participants: some distributions show more extreme excursions, while others are more tightly concentrated around moderate glucose levels. The distributions differ in both spread and tail behavior, supporting a distributional modeling approach.

Figure~\ref{fig:glucose_lineplot} reports Wasserstein errors for the \(20\) held-out participants. Prediction error varies substantially across participants for all methods, indicating considerable heterogeneity in the glucose distributions and suggesting that some subjects are intrinsically harder to predict. The error curves cross frequently, so no single method uniformly dominates. This highlights the importance of examining both patient-level performance and aggregate summaries.

Figure~\ref{fig:glucose_boxplot} summarizes the distribution of Wasserstein errors across the \(20\) test participants. Among competing methods, RF has the smallest median error, while DL and GFR yield comparable median performance and remain competitive overall. DFR shows a slightly larger median error, and KNN has the largest median. The interquartile ranges overlap substantially, indicating that differences across methods are moderate and depend on the individual patient. In particular, DL exhibits both very small and very large errors across subjects, suggesting that it adapts well in some cases but is more sensitive to subject-level heterogeneity than RF.

Overall, these results illustrate the potential of the proposed method while indicating that no single approach performs uniformly best across the two distribution-valued applications considered.

\begin{figure}[ht!]
    \centering
    \begin{minipage}[t]{0.48\textwidth}
        \centering
        \includegraphics[width=\textwidth]{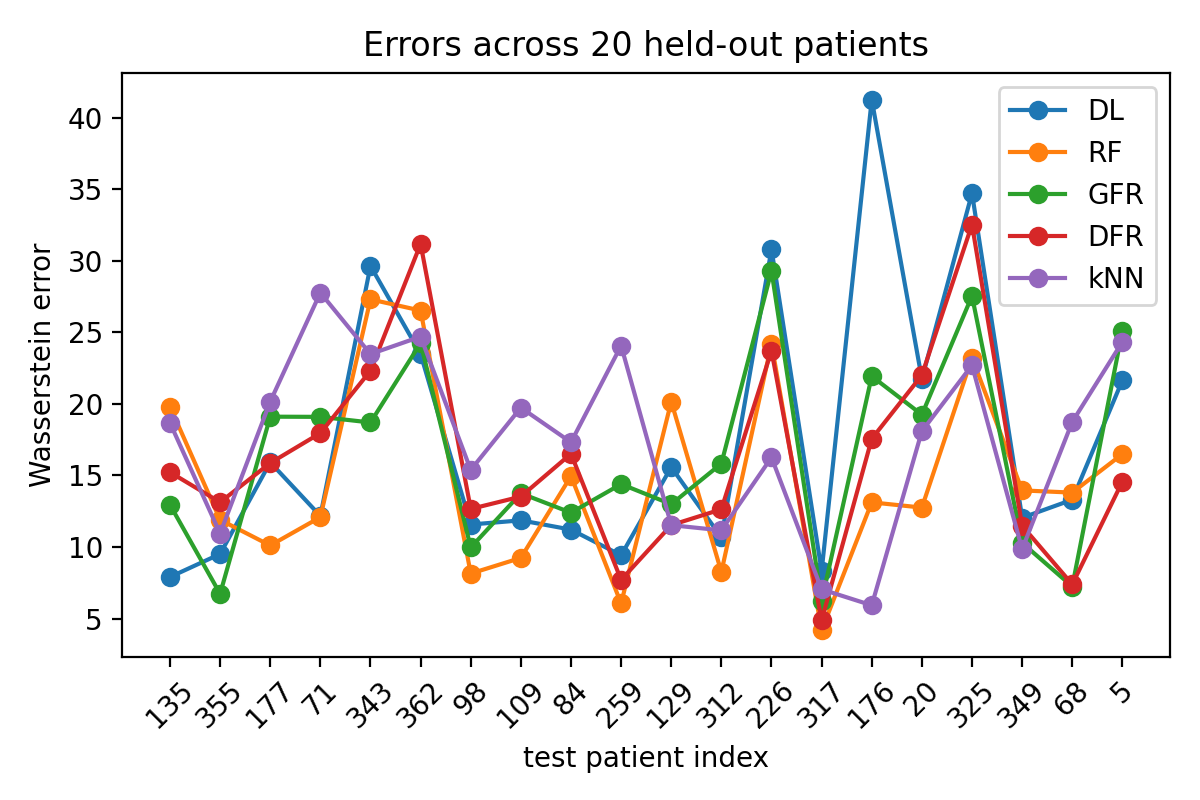}
        \caption{Glucose error line plot.}
        \label{fig:glucose_lineplot}
    \end{minipage}
    \hfill
    \begin{minipage}[t]{0.48\textwidth}
        \centering
        \includegraphics[width=\textwidth]{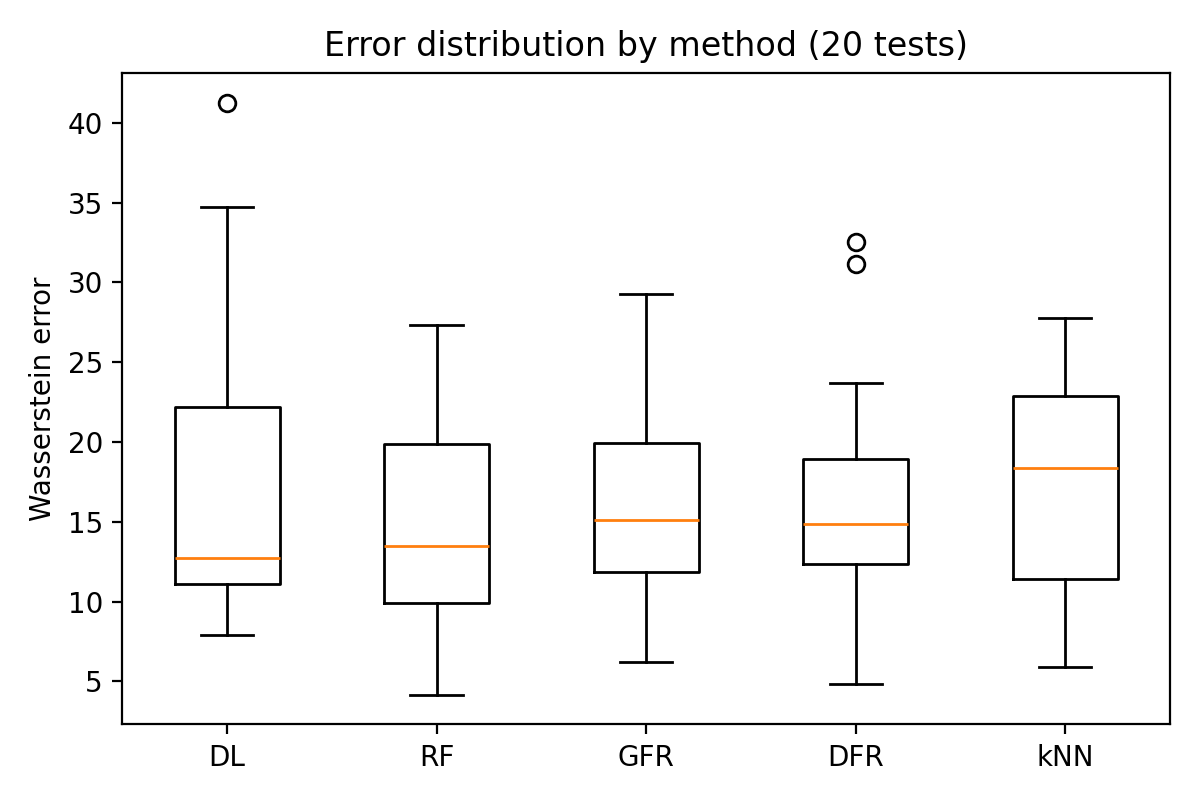}
        \caption{Glucose error boxplot.}
        \label{fig:glucose_boxplot}
    \end{minipage}
\end{figure}

\section{Final Remarks}\label{sec:dis}

In this paper, we introduce a deep learning framework for
regression with responses taking values in general metric
spaces, building on the anchor-based regression approach
of \cite{matabuena2026randomeffects}. Our estimator reduces
the original regression problem to a collection of
scalar-valued deep learning tasks that can be trained
independently and in parallel, without requiring an
embedding of the response space.

We establish excess-risk bounds for general ReLU neural
network classes and derive explicit convergence rates
under hierarchical composition assumptions
\cite{kohler2021rate}. These rates depend on the smoothness
and intrinsic dimensions of the component functions,
rather than directly on the ambient predictor dimension,
providing conditions under which the curse of
dimensionality can be mitigated. Relative to approaches based on learned manifold representations or Fr\'echet-regression output layers \cite{iao2025deep,zhou2026e2m,kim2026dfnn}, our framework analyzes a distinct anchor-based estimator. Its excess-risk rates are explicit under hierarchical-composition and anchor-coverage conditions and apply to bounded metric response spaces, including functional and distributional data.

We evaluate the empirical performance of the proposed method through simulation studies and applications to digital health data, where the outcomes are probability distributions. In the multivariate distributional setting, some existing deep learning approaches rely on isometric or approximately isometric Euclidean embeddings \cite{iao2025deep}. Such embeddings, however, need not exist for general metric spaces. Our method avoids this structural requirement by operating directly through the metric geometry of the response space. Across the simulation settings considered, the proposed approach is particularly competitive as the sample size and complexity of the regression problem increase, and it achieves lower test error than the competing approaches in several configurations. The digital health applications further illustrate its ability to accommodate complex distribution-valued outcomes arising from high-resolution physiological data.

The main limitation of the proposed approach is its computational cost, which results from fitting a potentially large collection of univariate regression models. Although these fits can be carried out in parallel, the overall computational burden may become substantial for large datasets. Future work will therefore focus on developing more scalable variants, including subsampling strategies that reduce the number of required model fits. Another important direction is uncertainty quantification \cite{lugosi2026conformalknnpredictiveuncertainty}, for example through conformal prediction methods adapted to metric-space-valued responses. Finally, we plan to extend the framework to settings with non-Euclidean predictors, potentially building on prototype-based representations and related universal regression approaches \cite{gyorfi2021universal}.

\section*{Acknowledgements}
Oscar Hernan Madrid Padilla was partially supported by the
National Science Foundation under CAREER Award DMS-2541747.

\newpage

\appendix

\allowdisplaybreaks[1]
%\DeclarePairedDelimiter\ceil{\lceil}{\rceil}
%\DeclarePairedDelimiter\floor{\lfloor}{\rfloor}

	%\section{Introduction}

\Needspace{8\baselineskip}
\section{Prediction-error consequence of Theorem~\ref{theo1} }

\begin{remark}
\label{rem1}
Suppose that $M$ satisfies the following growth condition around its
conditional minimizer: there exist a measurable function
$\Delta^2:\mathcal Y\times\mathcal Y\to[0,\infty)$ and a constant
$c_0>0$, independent of $x$ and $n$, such that
\begin{equation}
\label{eqn:e18}
M(x,y)-M(x,m(x))
\geq
c_0\Delta^2(m(x),y),
\qquad
x\in\mathcal X,\quad y\in\mathcal Y.
\end{equation}
This condition requires the conditional Fr\'echet risk to increase at
least at rate $\Delta^2(m(x),y)$ as $y$ moves away from the minimizer
$m(x)$.

Applying \eqref{eqn:e18} with $y=\widehat m(x)$ gives, pointwise in
$x$,
\[
c_0\Delta^2(m(x),\widehat m(x))
\leq
M(x,\widehat m(x))-M(x,m(x)).
\]
Conditional on the training data $\mathcal D_n$, the estimator
$\widehat m$ is fixed. Integrating the preceding inequality with
respect to $P_X$ therefore yields
\[
\begin{aligned}
c_0\mathbb E\!\left[
\Delta^2(m(X),\widehat m(X))
\mid\mathcal D_n
\right]
&=
c_0\int
\Delta^2(m(x),\widehat m(x))\,dP_X(x)\\
&\leq
\int
\{M(x,\widehat m(x))-M(x,m(x))\}\,dP_X(x)\\
&=
\mathcal R(\widehat m)-\mathcal R(m).
\end{aligned}
\]
Here the expectation is over a fresh covariate $X$, conditional on
the training data.

Consequently, the excess-risk bound in Theorem~\ref{theo1} implies,
with probability approaching one,
\[
\begin{split}
&\mathbb E\!\left[
\Delta^2(m(X),\widehat m(X))
\mid\mathcal D_n
\right]\\
&\quad\lesssim
r_n+
\sqrt{
\frac{
|I|\max\{\mathcal U_n^2,A_n^2\}
\max_{i\in I}\{L_i^2\nu_i^2\log(L_i\nu_i)\}
\log(A_nn)
}{n}
}
+\sqrt{|I|\phi_n},
\end{split}
\]
where the fixed factor $c_0^{-1}$ is absorbed into the implicit
constant. Likewise, Corollary~\ref{cor1} gives
\[
\mathbb E\!\left[
\Delta^2(m(X),\widehat m(X))
\mid\mathcal D_n
\right]
=
O_{\mathbb P}(r_n+\alpha_n).
\]

For Euclidean squared-distance loss, if the conditional mean lies in the response domain, then
\[
M(x,y)-M(x,m(x))=\|y-m(x)\|_2^2,
\]
so \eqref{eqn:e18} holds with $c_0=1$ and $\Delta=d$. This identity does not replace the boundedness and coverage assumptions of Theorem~\ref{theo1}. 

A metric-space example corresponding to scalar median regression is obtained by taking
\[
d(a,b)=\sqrt{|a-b|/2},
\qquad a,b\in\mathcal Y\subset\mathbb R.
\]
Under standard local density conditions ensuring uniqueness of the conditional median and quadratic growth of the conditional median loss, the growth condition \eqref{eqn:e18} holds with
\[
\Delta^2(a,b)=\min\{|a-b|,|a-b|^2\}.
\]
Related growth conditions for quantile regression are studied in \cite{madrid2022risk,padilla2022quantile}. The present metric formulation corresponds specifically to the symmetric median case.
\end{remark}

\section{ReLU neural network class}
\label{sec:neural_network_class}

For a network with $L$ hidden layers of widths $w_1,\ldots,w_L$, consider a multilayer feedforward neural network $f:\mathbb{R}^d\to\mathbb{R}$ defined by
\begin{align}
f(x) = \sum_{i=1}^{w_L} c_{1,i}^{(L)} f_i^{(L)}(x) + c_{1,0}^{(L)}, \label{eq:form of approximation function 0}
\end{align}
\noindent where $c_{1,0}^{(L)}, \ldots, c_{1,w_L}^{(L)} \in \mathbb{R}$ are the output weights and the functions $f_i^{(L)}$ are defined recursively as
\begin{align}
f_i^{(s)}(x) = \tau\left( \sum_{j=1}^{w_{s-1}} c_{i,j}^{(s-1)} f_j^{(s-1)}(x) + c_{i,0}^{(s-1)} \right), \quad s \in \{2, \ldots, L\}, \label{eq:form of approximation function L}
\end{align}
\noindent with weights $c_{i,0}^{(s-1)}, \ldots, c_{i,w_{s-1}}^{(s-1)} \in \mathbb{R}$. The first hidden layer is given by
\begin{align}
f_i^{(1)}(x) = \tau\left( \sum_{j=1}^d c_{i,j}^{(0)} x_j + c_{i,0}^{(0)} \right), \label{eq:form of approximation function L2}
\end{align}
\noindent for weights $c_{i,0}^{(0)}, \ldots, c_{i,d}^{(0)} \in \mathbb{R}$.

Then, we denote by $\mathcal{F}(L,\nu)$ the functions of the form (\ref{eq:form of approximation function 0}) with $w_1  = \ldots = w_L =\nu$.

\section{Hierarchical composition models}
\label{sec:hierchical}

\begin{definition}[$(p, C)$-smoothness]\label{definition: p,c smoothness}
Let $p=q+s$ for some $q \in \mathbb{N}=\mathbb{Z}^{+}\cup\left\{0\right\}$ and $0<s \leq 1$. A function $g: \mathbb{R}^d \rightarrow \mathbb{R}$ is called $(p, C)$-smooth, if for every $\alpha=\left(\alpha_1, \ldots, \alpha_d\right) \in \mathbb{N}^d$, where $d \in \mathbb{Z}^{+}$ , with $\sum_{j=1}^d \alpha_j=q$ the partial derivative $\partial^q g /\left(\partial a_1^{\alpha_1} \ldots \partial a_d^{\alpha_d}\right)$ exists and satisfies
$$
\left|\frac{\partial^q g}{\partial a_1^{\alpha_1} \ldots \partial a_d^{\alpha_d}}\left(a\right)-\frac{\partial^q g}{\partial a_1^{\alpha_1} \ldots \partial a_d^{\alpha_d}}\left(b\right)\right| \leq C\| a- b\|^s
$$
for all $a, b \in \mathbb{R}^d$.% where $\|\cdot\|$ denotes the Euclidean norm.	
\end{definition}

Next, we recall the definition of hierarchical composition models from \cite{kohler2021rate}.

\begin{definition} [Space of Hierarchical Composition Models, \cite{kohler2021rate}]\label{definition: hierarchical composition model}
For $l=1$ and a smoothness constraint $\mathcal{P}\subseteq(0,\infty)\times\mathbb{N}$, define the hierarchical composition class by
$$
\begin{array}{lll}
         \mathcal{H}(1, \mathcal{P})  & := &\left\{h: \mathbb{R}^d \rightarrow \mathbb{R}: h(a )=g\left(a_{(\pi(1))}, \ldots, a_{(\pi(K))}\right),\right. \text { where }   g: \mathbb{R}^K \rightarrow \mathbb{R} \text { is }\\
    & & \,\,\,(p, C) \text {-smooth for some }(p, K) \in \mathcal{P} \text { and } \pi:\{1, \ldots, K\} \rightarrow\{1, \ldots, d\}\} .
    \end{array}
$$
For $l>1$, we recursively define
$$
\begin{array}{lll}
 \mathcal{H}(l, \mathcal{P})    &:=&  \left\{h: \mathbb{R}^d \rightarrow \mathbb{R}: h(a)=g\left(f_1(a), \ldots, f_K(a)\right),\right. \text { where }  g: \mathbb{R}^K \rightarrow \mathbb{R} \text { is }\\
     && \,\,\,(p, C) \text {-smooth for some }(p, K) \in \mathcal{P}\text { and } \left.f_i \in \mathcal{H}(l-1, \mathcal{P})\right\}.
\end{array}
%\begin{aligned}
%& 
%& g: \mathbb{R}^K \rightarrow \mathbb{R} \text { is }(p, C) \text {-smooth for some }(p, K) \in \mathcal{P} \\
%&\text { and } \left.f_i \in \mathcal{H}(l-1, \mathcal{P})\right\}
%\end{aligned}
$$	
\end{definition}

We now state a more formal version of Theorem \ref{theo2}.

\begin{theorem}\label{v2:theo2}

Suppose Assumption~\ref{as1} holds. For every $y\in\mathcal Y$, assume
$M(\cdot,y)\in\mathcal H(l(y),\mathcal P(y))$, with a chosen representation whose finite, nonempty set of component pairs is $\mathcal P(y)\subset[1,\infty)\times\mathbb N$.
Each component $g$ is $(t_g,C)$-smooth in the sense of Definition~\ref{definition: p,c smoothness}, where $t_g=q_g+s_g$, $q_g\ge0$, $s_g\in(0,1]$, and $(t_g,K_g)\in\mathcal P(y)$. The smoothness constant $C$, the bounds $c_2$ on all derivatives of orders at most $q_g$ (including order zero), and the Lipschitz constant $C_{\mathrm{Lip}}\ge1$ are uniform in $g$ and $y$. Assume
\[
\sup_y l(y)<\infty,\qquad
\sup_y\max_{(t,K)\in\mathcal P(y)}\max\{t,K\}<\infty.
\]
Let
\[
\phi_n:=\sup_{y\in\mathcal Y}\max_{(t,K)\in\mathcal P(y)}n^{-2t/(2t+K)}.
\]

There exist sufficiently large constants $c_3,c_4>0$, independent of $n$ and $y$, for which either architecture
\begin{equation}
\label{eqn:v2_choice1}
L(y)=\lceil c_3\log n\rceil,\qquad
\nu(y)=\left\lceil c_4\max_{(t,K)\in\mathcal P(y)}n^{K/[2(2t+K)]}\right\rceil
\end{equation}
or
\begin{equation}
\label{eqn:v2_choice2}
L(y)=\left\lceil c_3\max_{(t,K)\in\mathcal P(y)}n^{K/[2(2t+K)]}\log n\right\rceil,
\qquad \nu(y)=\lceil c_4\rceil
\end{equation}
gives, assuming the resulting architecture maps are measurable and $A_n$ satisfies \eqref{eqn:an}, the bound

\[
\mathcal R(\widehat m)\le\mathcal R(m)+r_n+
 c_1\sqrt{|I|\max\{\mathcal U_n^2,A_n^2\}\phi_n
             (\log n)^3\log(A_nn)}
\]

with probability approaching one, provided $r_n\downarrow0$ and
\begin{equation}
\label{v2_cond2}
C_n\sup_{y\in\mathcal Y}\min_{i\in I}d(Y_i,y)=o_{\mathbb P}(r_n).
\end{equation}
Here $\mathcal U_n$ and the allowable dependence of constants are as in Theorems~\ref{theo1} and~\ref{theo2}.
\end{theorem}

\section{Proofs}

\subsection{Proof of Theorem \ref{theo1}}

\begin{proof}
{Step 1: deterministic risk decomposition.}
Fix the training sample and set
\[
k:\mathcal X\to I,\qquad k(x)=\min\argmin_{i\in I}\widehat M_i(x).
\]
Thus $\widehat m(x)=Y_{k(x)}$. Since $m(x)$ minimizes $M(x,\cdot)$,
\begin{equation}
\label{eqn:e1}
\begin{split}
\mathcal R(\widehat m)
&=\mathcal R(m)+\int\{M(x,Y_{k(x)})-M(x,m(x))\}\,dP_X(x)\\
&\le\mathcal R(m)+\int|M(x,Y_{k(x)})-\widehat M_{k(x)}(x)|\,dP_X(x)\\
&\hspace{3.7em}+\int|\widehat M_{k(x)}(x)-M(x,m(x))|\,dP_X(x)\\
&=: \mathcal R(m)+T_1+T_2.
\end{split}
\end{equation}
By Cauchy--Schwarz and the pointwise inequality $\max_i a_i^2\le\sum_i a_i^2$,
\begin{equation}
\label{eqn:e2}
\begin{split}
T_1
&\le\left(\int\max_{i\in I}|M_i(x)-\widehat M_i(x)|^2\,dP_X(x)\right)^{1/2}\\
&\le\left(\sum_{i\in I}\int|M_i(x)-\widehat M_i(x)|^2\,dP_X(x)\right)^{1/2}\\
&\le\sqrt{|I|}\max_{i\in I}\|M_i-\widehat M_i\|_{\mathcal L_2}.
\end{split}
\end{equation}
The elementary inequality $|\min_i a_i-\min_i b_i|\le\max_i|a_i-b_i|$ gives
\begin{equation}
\label{eqn:e3}
\begin{split}
T_2
&\le\int\max_{i\in I}|\widehat M_i(x)-M_i(x)|\,dP_X(x)\\
&\qquad+\int\left\{\min_{i\in I}M(x,Y_i)-M(x,m(x))\right\}\,dP_X(x)\\
&\le\sqrt{|I|}\max_{i\in I}\|M_i-\widehat M_i\|_{\mathcal L_2}\\
&\qquad+\int\left\{\min_{i\in I}M(x,Y_i)-M(x,m(x))\right\}\,dP_X(x).
\end{split}
\end{equation}
To obtain the second inequality in \eqref{eqn:e3}, we bound the first
integral on the right-hand side. Define
\[
e_i(x):=\widehat{M}_i(x)-M_i(x),
\qquad i\in I.
\]
Since $P_X$ is a probability measure, the Cauchy--Schwarz inequality
implies
\begin{align}
\int \max_{i\in I}|e_i(x)|\,dP_X(x)
&\leq
\left(
\int
\left[\max_{i\in I}|e_i(x)|\right]^2
\,dP_X(x)
\right)^{1/2}.
\label{eq:T2-CS}
\end{align}
For every fixed $x\in\mathcal X$, we have
\[
\left[\max_{i\in I}|e_i(x)|\right]^2
=
\max_{i\in I}|e_i(x)|^2
\leq
\sum_{i\in I}|e_i(x)|^2,
\]
because all terms in the sum are nonnegative. Therefore,
\begin{align}
\int \max_{i\in I}|e_i(x)|\,dP_X(x)
&\leq
\left(
\int
\sum_{i\in I}|e_i(x)|^2
\,dP_X(x)
\right)^{1/2}
\nonumber\\
&=
\left(
\sum_{i\in I}
\int |e_i(x)|^2\,dP_X(x)
\right)^{1/2}
\nonumber\\
&=
\left(
\sum_{i\in I}
\|e_i\|_{\mathcal L_2}^2
\right)^{1/2}.
\label{eq:T2-sum-L2}
\end{align}
Next, for every $i\in I$,
\[
\|e_i\|_{\mathcal L_2}^2
\leq
\max_{j\in I}\|e_j\|_{\mathcal L_2}^2.
\]
Hence,
\begin{align}
\left(
\sum_{i\in I}
\|e_i\|_{\mathcal L_2}^2
\right)^{1/2}
&\leq
\left(
|I|
\max_{i\in I}
\|e_i\|_{\mathcal L_2}^2
\right)^{1/2}
\nonumber\\
&=
\sqrt{|I|}
\max_{i\in I}
\|e_i\|_{\mathcal L_2}.
\label{eq:T2-max-L2}
\end{align}
Combining \eqref{eq:T2-CS}--\eqref{eq:T2-max-L2}, and recalling that
$e_i=\widehat{M}_i-M_i$, gives
\begin{equation}
\label{eq:T2-estimation-bound}
\int
\max_{i\in I}
|\widehat{M}_i(x)-M_i(x)|
\,dP_X(x)
\leq
\sqrt{|I|}
\max_{i\in I}
\|\widehat{M}_i-M_i\|_{\mathcal L_2}.
\end{equation}
Substituting this bound into the preceding inequality for $T_2$ yields
\begin{equation}
\label{eqn:e3-refined}
\begin{split}
T_2
&\leq
\sqrt{|I|}
\max_{i\in I}
\|M_i-\widehat{M}_i\|_{\mathcal L_2}\\
&\qquad+
\int
\left\{
\min_{i\in I}M(x,Y_i)
-
M(x,m(x))
\right\}
\,dP_X(x).
\end{split}
\end{equation}

\medskip
{Step 2: coverage under Assumption~\ref{as1}.}
For any $a,b\in\mathcal Y$, the triangle inequality yields
\[
\begin{split}
|M(x,a)-M(x,b)|
&\le\mathbb E\!\left[|d(Y,a)-d(Y,b)|\{d(Y,a)+d(Y,b)\}\mid X=x\right]\\
&\le2C_n d(a,b).
\end{split}
\]
Consequently,
\begin{equation*}
\tag{\theequation a}\label{eqn:e5}
\begin{split}
\min_{i\in I}M(x,Y_i)
&\le\inf_{y\in\mathcal Y}\left\{M(x,y)+2C_n\min_{i\in I}d(Y_i,y)\right\}\\
&\le M(x,m(x))+2C_n\sup_{y\in\mathcal Y}\min_{i\in I}d(Y_i,y).
\end{split}
\end{equation*}
Substituting into \eqref{eqn:e3-refined} gives
\begin{equation}
\label{eqn:e6}
T_2\le\sqrt{|I|}\max_{i\in I}\|M_i-\widehat M_i\|_{\mathcal L_2}
       +2C_n\sup_{y\in\mathcal Y}\min_{i\in I}d(Y_i,y).
\end{equation}
Combining \eqref{eqn:e1}, \eqref{eqn:e2}, and \eqref{eqn:e6},
\begin{equation}
\label{eqn:e7}
\mathcal R(\widehat m)-\mathcal R(m)
\le2\sqrt{|I|}\max_{i\in I}\|M_i-\widehat M_i\|_{\mathcal L_2}
       +2C_n\sup_{y\in\mathcal Y}\min_{i\in I}d(Y_i,y).
\end{equation}

\medskip
{Step 3: conditional network estimation under Assumption~\ref{as1}.}
Fix $i$ and condition on $Y_i$. The $n-1$ pairs
$(X_j,d^2(Y_j,Y_i))$, $j\ne i$, are i.i.d., with regression function $M_i$ and architecture fixed by $Y_i$. Since both the response and its conditional mean lie in $[0,C_n^2]$,
\[
|d^2(Y_j,Y_i)-M_i(X_j)|\le C_n^2=\mathcal U_n.
\]
Moreover, \eqref{eqn:an} implies $d^2(Y_j,Y_i)\le A_n/8$.
We apply the bounded-response oracle bound proved in Theorem~\ref{thm1_v2} below, which makes the finite-sample probability and normalization explicit.

For the fixed predictor dimension $p$, the dense-network pseudodimension bound in \cite{bartlett2019nearly} and the associated empirical covering bound \cite{gyorfi2002distribution} imply the following. Define the normalized clipped class
\[
\mathcal F_{A_n}(L_i,\nu_i):=
 \{f^{A_n}/A_n:f\in\mathcal F(L_i,\nu_i)\}.
\]
After increasing $C_0$ if necessary, for all $0<\delta<1$,
\begin{equation}
\label{eqn:entropy}
\begin{split}
\sup_{x_1,\ldots,x_n\in[0,1]^p}
\log N(\delta,\mathcal F_{A_n}(L_i,\nu_i),\|\cdot\|_n)
&\le C_0L_i^2\nu_i^2\log(L_i\nu_i)\log(2A_n/\delta)\\
&=: \eta_{i,n}(\delta).
\end{split}
\end{equation}
The same upper bound holds for
$\{(f^{A_n}-M_i)/(2A_n):f\in\mathcal F(L_i,\nu_i)\}$,
because translation leaves covering numbers unchanged and contraction cannot increase them. All constants here are uniform in the conditioned anchor. Write
$q_i=C_0L_i^2\nu_i^2\log(L_i\nu_i)$ only as shorthand in the next calculation. Recall that
\[
\eta_{i,n}(\delta)
=
q_i\log\left(\frac{2A_n}{\delta}\right),
\qquad
q_i
:=
C_0L_i^2\nu_i^2\log(L_i\nu_i).
\]
We evaluate the two dyadic sums in
\eqref{eqn:entropy_0_v2} explicitly.

For the first sum,
\[
\eta_{i,n}(2^{-l-l'-1})
=
q_i\left\{
\log(2A_n)+(l+l'+1)\log2
\right\}.
\]
Consequently,
\[
e^{-C_1\eta_{i,n}(2^{-l-l'-1})}
=
(2A_n)^{-C_1q_i}
2^{-C_1q_i(l+l'+1)}.
\]
Let
\[
r_1:=2^{-C_1q_i}\in(0,1).
\]
Then
\begin{align}
&\sum_{l=0}^\infty\sum_{l'=1}^\infty
e^{-C_1\eta_{i,n}(2^{-l-l'-1})}
\nonumber\\
&\qquad=
(2A_n)^{-C_1q_i}
\sum_{l=0}^\infty\sum_{l'=1}^\infty
r_1^{l+l'+1}
\nonumber\\
&\qquad=
(2A_n)^{-C_1q_i}
\left(\sum_{l=0}^\infty r_1^l\right)
\left(\sum_{l'=1}^\infty r_1^{l'+1}\right)
\nonumber\\
&\qquad=
(2A_n)^{-C_1q_i}
\frac{1}{1-r_1}
\frac{r_1^2}{1-r_1}
\nonumber\\
&\qquad=
\frac{
(2A_n)^{-C_1q_i}2^{-2C_1q_i}
}{
(1-2^{-C_1q_i})^2
}.
\label{eq:dyadic-first-sum}
\end{align}

Similarly,
\[
\eta_{i,n}(2^{-l-1})
=
q_i\left\{
\log(2A_n)+(l+1)\log2
\right\},
\]
so
\[
e^{-C_2\eta_{i,n}(2^{-l-1})}
=
(2A_n)^{-C_2q_i}
2^{-C_2q_i(l+1)}.
\]
Letting $r_2:=2^{-C_2q_i}$, we obtain
\begin{align}
\sum_{l=0}^\infty
e^{-C_2\eta_{i,n}(2^{-l-1})}
&=
(2A_n)^{-C_2q_i}
\sum_{l=0}^\infty r_2^{l+1}
\nonumber\\
&=
(2A_n)^{-C_2q_i}
\frac{r_2}{1-r_2}
\nonumber\\
&=
\frac{
(2A_n)^{-C_2q_i}2^{-C_2q_i}
}{
1-2^{-C_2q_i}
}.
\label{eq:dyadic-second-sum}
\end{align}

Combining \eqref{eq:dyadic-first-sum} and
\eqref{eq:dyadic-second-sum} gives
\begin{equation}
\label{eqn:ver5}
\begin{split}
&\sum_{l=0}^\infty\sum_{l'=1}^\infty
e^{-C_1\eta_{i,n}(2^{-l-l'-1})}
+
\sum_{l=0}^\infty
e^{-C_2\eta_{i,n}(2^{-l-1})}
\\
&\qquad=
\frac{
(2A_n)^{-C_1q_i}2^{-2C_1q_i}
}{
(1-2^{-C_1q_i})^2
}
+
\frac{
(2A_n)^{-C_2q_i}2^{-C_2q_i}
}{
1-2^{-C_2q_i}
}.
\end{split}
\end{equation}

The constants $C_1$ and $C_2$ appearing here are the positive
constants in the entropy condition \eqref{eqn:entropy_0_v2} of
Theorem~\ref{thm1_v2}.

We next verify that the right-hand side of \eqref{eqn:ver5} converges
to zero. For $j\in\{1,2\}$,
\[
(2A_n)^{-C_jq_i}
=
\exp\{-C_jq_i\log(2A_n)\}.
\]
Hence, if
\[
q_i\log(2A_n)\longrightarrow\infty,
\]
then
\[
(2A_n)^{-C_jq_i}\longrightarrow0.
\]
Moreover, under the architecture choices considered below, $q_i$ is
eventually bounded away from zero (and in fact diverges). Therefore,
for $j\in\{1,2\}$,
\[
1-2^{-C_jq_i}
\]
is bounded away from zero for all sufficiently large $n$. It follows
that
\[
\frac{(2A_n)^{-C_1q_i}2^{-2C_1q_i}}
     {(1-2^{-C_1q_i})^2}
\longrightarrow0
\]
and
\[
\frac{(2A_n)^{-C_2q_i}2^{-C_2q_i}}
     {1-2^{-C_2q_i}}
\longrightarrow0.
\]
Thus the entire expression in \eqref{eqn:ver5} converges to zero.

We emphasize that this calculation is exact. In particular,
\[
\eta_{i,n}(2^{-l-l'-1})
=
q_i\left\{
\log(2A_n)+(l+l'+1)\log2
\right\},
\]
and therefore
\[
\exp\{-C_1\eta_{i,n}(2^{-l-l'-1})\}
=
(2A_n)^{-C_1q_i}
2^{-C_1q_i(l+l'+1)}.
\]
Thus the two factors arise from the identity
$\exp\{-(a+b)\}=\exp(-a)\exp(-b)$. For every $l\geq0$, using
\[
\eta_{i,n}(\delta)
=
q_i\log\left(\frac{2A_n}{\delta}\right),
\]
we have
\[
\eta_{i,n}(2^{-l-l'})
=
q_i\left\{
\log(2A_n)+(l+l')\log2
\right\}
\]
and
\[
\eta_{i,n}(2^{-l})
=
q_i\left\{
\log(2A_n)+l\log2
\right\}.
\]
Hence, writing
\[
D_l:=\log(2A_n)+l\log2,
\]
the factor $q_i$ cancels and
\begin{align}
\sum_{l'=1}^\infty
\frac{\eta_{i,n}(2^{-l-l'})}
{2^{2l'}\eta_{i,n}(2^{-l})}
&=
\sum_{l'=1}^\infty
\frac{D_l+l'\log2}{4^{l'}D_l}
\nonumber\\
&=
\sum_{l'=1}^\infty\frac1{4^{l'}}
+
\frac{\log2}{D_l}
\sum_{l'=1}^\infty\frac{l'}{4^{l'}}
\nonumber\\
&=
\frac13+
\frac{4\log2}{9D_l}
\nonumber\\
&=
\frac13+
\frac{4\log2}
{9\{\log(2A_n)+l\log2\}}.
\label{eq:dyadic-regularity-calculation}
\end{align}
The identities
\[
\sum_{j=1}^\infty4^{-j}=\frac13,
\qquad
\sum_{j=1}^\infty j4^{-j}=\frac49
\]
were used in the third line. Since $A_n\ge8$ and $l\ge0$,
\[
\log(2A_n)+l\log2
\ge
\log16
=
4\log2.
\]
Consequently,
\[
\sum_{l'=1}^\infty
\frac{\eta_{i,n}(2^{-l-l'})}
{2^{2l'}\eta_{i,n}(2^{-l})}
\le
\frac13+\frac19
=
\frac49
<1,
\qquad l\ge0.
\]
Thus condition \eqref{eqn:cond_sum_v2} holds uniformly in $l$.

Fix $i\in I$ and condition on $Y_i$. The regression sample then has
size
\[
N:=n-1,
\]
with true regression function
\[
f^*=M_i=M(\cdot,Y_i).
\]
In Theorem~\ref{thm1_v2}, take the comparison function to be
\[
\bar f=\widetilde M(\cdot,Y_i),
\]
and define its squared approximation error by
\[
\phi_{i,n}
:=
\|\widetilde M(\cdot,Y_i)-M(\cdot,Y_i)\|_\infty^2.
\]

Recall that
\[
\eta_{i,n}(\delta)
=
q_i\log\left(\frac{2A_n}{\delta}\right),
\qquad
q_i
=
C_0L_i^2\nu_i^2\log(L_i\nu_i).
\]
Evaluating this entropy bound at the scale appearing in
Theorem~\ref{thm1_v2} gives
\[
\eta_{i,n}\left(\frac{1}{24N}\right)
=
q_i\log(48A_nN).
\]
Since $N=n-1\asymp n$, this implies
\[
\eta_{i,n}\left(\frac{1}{24N}\right)
\lesssim
L_i^2\nu_i^2\log(L_i\nu_i)\log(A_nn).
\]
Hence the critical-radius upper bound in
Theorem~\ref{thm1_v2} may be chosen so that
\[
\delta_i^2
\lesssim
\frac{\log n}{n-1}
+
\frac{
L_i^2\nu_i^2\log(L_i\nu_i)\log(A_nn)
}{n-1}.
\]

Applying Theorem~\ref{thm1_v2} with sample size $N=n-1$ therefore gives
\begin{equation}
\label{eq:conditional-L2-first}
\|M_i-\widehat M_i\|_{\mathcal L_2}^2
\leq
C\left\{
\phi_{i,n}
+
A_n^2\delta_i^2
\right\}
\end{equation}
with conditional failure probability at most $C_4n^{-3}$.
Substituting the preceding bound for $\delta_i^2$ into
\eqref{eq:conditional-L2-first} yields
\[
\begin{split}
\|M_i-\widehat M_i\|_{\mathcal L_2}^2
\leq C\Bigg[
&\phi_{i,n}
+
\frac{A_n^2\log n}{n-1}\\
&+
\frac{
A_n^2L_i^2\nu_i^2
\log(L_i\nu_i)\log(A_nn)
}{n-1}
\Bigg].
\end{split}
\]

The middle term is dominated by the last term. Indeed,
$L_i\nu_i\geq2$ implies
\[
L_i^2\nu_i^2\log(L_i\nu_i)
\geq
4\log2,
\]
and $A_n\geq8$ implies
\[
\log(A_nn)\geq\log n.
\]
Consequently,
\[
\frac{A_n^2\log n}{n-1}
\lesssim
\frac{
A_n^2L_i^2\nu_i^2
\log(L_i\nu_i)\log(A_nn)
}{n-1}.
\]
Since $(n-1)^{-1}\asymp n^{-1}$, we obtain
\[
\begin{split}
\|M_i-\widehat M_i\|_{\mathcal L_2}^2
\leq C_3\Bigg[
&
\frac{
A_n^2L_i^2\nu_i^2
\log(L_i\nu_i)\log(A_nn)
}{n}\\
&+
\|\widetilde M(\cdot,Y_i)-M(\cdot,Y_i)\|_\infty^2
\Bigg].
\end{split}
\]

For each $i\in I$, let $E_i$ denote the event on which the preceding
conditional estimation bound holds. Since the conditional failure
probability is bounded uniformly in the conditioned value of $Y_i$,
\[
\mathbb P(E_i^c\mid Y_i)\leq C_4n^{-3}
\qquad\text{a.s.}
\]
Taking expectation with respect to $Y_i$ and using the tower property
gives
\[
\mathbb P(E_i^c)
=
\mathbb E\!\left[
\mathbb P(E_i^c\mid Y_i)
\right]
\leq C_4n^{-3}.
\]
Consequently, by the union bound,
\[
\mathbb P\left(\bigcup_{i\in I}E_i^c\right)
\leq
\sum_{i\in I}\mathbb P(E_i^c)
\leq
C_4|I|n^{-3}.
\]
No independence among the events $\{E_i:i\in I\}$ is needed for this
argument. Since $I\subseteq[n]$, we have $|I|\leq n$, and hence
\[
C_4|I|n^{-3}\leq C_4n^{-2}.
\]

Moreover, by the definition of $\phi_n$,
\[
\|\widetilde M(\cdot,Y_i)-M(\cdot,Y_i)\|_\infty^2
\leq\phi_n
\qquad\text{for every }i\in I.
\]
Also,
\[
L_i^2\nu_i^2\log(L_i\nu_i)
\leq
\max_{k\in I}
\{L_k^2\nu_k^2\log(L_k\nu_k)\}.
\]
Therefore, on $\bigcap_{i\in I}E_i$,
\[
\begin{split}
\max_{i\in I}\|M_i-\widehat M_i\|_{\mathcal L_2}^2
\leq C_3\Bigg[
&
\frac{
\max\{\mathcal U_n^2,A_n^2\}
\max_{i\in I}
\{L_i^2\nu_i^2\log(L_i\nu_i)\}
\log(A_nn)
}{n}
+\phi_n
\Bigg].
\end{split}
\]
Thus, defining
\[
\begin{split}
\Omega_1:=\Bigg\{
\max_{i\in I}\|M_i-\widehat M_i\|_{\mathcal L_2}^2
\leq C_3\Bigg[
&
\frac{
\max\{\mathcal U_n^2,A_n^2\}
\max_{i\in I}
\{L_i^2\nu_i^2\log(L_i\nu_i)\}
\log(A_nn)
}{n}
+\phi_n
\Bigg]
\Bigg\},
\end{split}
\]
we have
\[
\mathbb P(\Omega_1)
\geq
1-C_4|I|n^{-3}
\geq
1-C_4n^{-2}.
\]

The additional growth condition
\begin{equation}
\label{eqn:cond1}
\frac{
\min_{i\in I}
\{L_i^2\nu_i^2\log(L_i\nu_i)\}
\log(2A_n^2)
}{
\log n
}
\xrightarrow{\mathbb P}\infty
\end{equation}
ensures that the preceding entropy regularity conditions hold
uniformly over $i\in I$. Indeed, with
\[
q_i=C_0L_i^2\nu_i^2\log(L_i\nu_i),
\]
the dyadic expressions in \eqref{eqn:ver5} contain factors of the
form
\[
(2A_n)^{-Cq_i}
=
\exp\{-Cq_i\log(2A_n)\}.
\]
Since $A_n\geq8$,
\[
\log(2A_n)\geq\frac12\log(2A_n^2),
\]
and hence \eqref{eqn:cond1} implies
\[
\min_{i\in I}q_i\log(2A_n)\longrightarrow\infty
\]
in probability. Thus the dyadic terms converge to zero uniformly over
the anchors.

We now substitute the simultaneous estimation bound into
\eqref{eqn:e7}. On $\Omega_1$, writing
\[
B_n:=
\frac{
\max\{\mathcal U_n^2,A_n^2\}
\max_{i\in I}
\{L_i^2\nu_i^2\log(L_i\nu_i)\}
\log(A_nn)
}{n},
\]
we have
\[
\max_{i\in I}\|M_i-\widehat M_i\|_{\mathcal L_2}
\leq
C_3^{1/2}\sqrt{B_n+\phi_n}
\leq
C_3^{1/2}
\left\{\sqrt{B_n}+\sqrt{\phi_n}\right\},
\]
where the last inequality uses $\sqrt{a+b}\leq\sqrt a+\sqrt b$.
Therefore, \eqref{eqn:e7} yields
\begin{equation}
\label{eqn:e13}
\begin{split}
\mathcal R(\widehat m)-\mathcal R(m)
\leq{}&
2C_3^{1/2}\Bigg[
\sqrt{
\frac{
|I|\max\{\mathcal U_n^2,A_n^2\}
\max_{i\in I}
\{L_i^2\nu_i^2\log(L_i\nu_i)\}
\log(A_nn)
}{n}
}
+\sqrt{|I|\phi_n}
\Bigg]\\
&+
2C_n\sup_{y\in\mathcal Y}\min_{i\in I}d(Y_i,y).
\end{split}
\end{equation}

Finally, condition \eqref{cond2} implies
\[
\frac{
C_n\sup_{y\in\mathcal Y}\min_{i\in I}d(Y_i,y)
}{r_n}
\xrightarrow{\mathbb P}0.
\]
Hence
\[
2C_n\sup_{y\in\mathcal Y}\min_{i\in I}d(Y_i,y)
\leq r_n
\]
with probability approaching one. Since
$\mathbb P(\Omega_1)\to1$, combining these two events proves the
theorem.

\end{proof}

\subsection{Proof of Corollary \ref{cor1}}

\begin{proof}
With the notation of the proof of Theorem \ref{theo1}, the pathwise decomposition  (\ref{eqn:e1}) is unchanged:
\begin{equation}
\label{eqn:e16}
\begin{split}
\mathcal R(\widehat m)
&\le\mathcal R(m)+\int|M(x,Y_{k(x)})-\widehat M_{k(x)}(x)|\,dP_X(x)\\
&\qquad+\int|\widehat M_{k(x)}(x)-M(x,m(x))|\,dP_X(x)\\
&=: \mathcal R(m)+T_1+T_2.
\end{split}
\end{equation}
Condition \eqref{eqn:cond3} implies
\begin{equation}
\label{eqn:e18-corollary}
\begin{split}
T_1&\le\int\max_{i\in I}|\widehat M_i(x)-M_i(x)|\,dP_X(x)\\
&\le\max_{i\in I}\|\widehat M_i-M_i\|_\infty=O_{\mathbb P}(\alpha_n),
\end{split}
\end{equation}
not a deterministic bound by $\alpha_n$. Similarly, since
\[
k(x)\in\arg\min_{i\in I}\widehat M_i(x),
\]
we have
\[
\widehat M_{k(x)}(x)
=
\min_{i\in I}\widehat M_i(x).
\]
Therefore,
\[
T_2
=
\int
\left|
\min_{i\in I}\widehat M_i(x)
-
M(x,m(x))
\right|
\,dP_X(x).
\]
For each fixed $x$, add and subtract
$\min_{i\in I}M_i(x)$ and apply the triangle inequality:
\[
\begin{split}
&
\left|
\min_{i\in I}\widehat M_i(x)
-
M(x,m(x))
\right|
\\
&\quad\le
\left|
\min_{i\in I}\widehat M_i(x)
-
\min_{i\in I}M_i(x)
\right|
\\
&\qquad\quad+
\left|
\min_{i\in I}M_i(x)
-
M(x,m(x))
\right|.
\end{split}
\]
The elementary inequality
\[
\left|\min_i a_i-\min_i b_i\right|
\leq
\max_i|a_i-b_i|
\]
implies
\[
\left|
\min_{i\in I}\widehat M_i(x)
-
\min_{i\in I}M_i(x)
\right|
\leq
\max_{i\in I}
|\widehat M_i(x)-M_i(x)|.
\]
Moreover, since $m(x)$ minimizes $M(x,\cdot)$ over all
$y\in\mathcal Y$,
\[
M(x,m(x))
\leq
M(x,Y_i)
=
M_i(x)
\qquad\text{for every }i\in I.
\]
Hence
\[
\min_{i\in I}M_i(x)-M(x,m(x))\geq0,
\]
and therefore
\[
\left|
\min_{i\in I}M_i(x)-M(x,m(x))
\right|
=
\min_{i\in I}M_i(x)-M(x,m(x)).
\]
Combining these inequalities gives
\[
\begin{split}
T_2
\le{}&
\int
\max_{i\in I}
|\widehat M_i(x)-M_i(x)|
\,dP_X(x)
\\
&+
\int
\left\{
\min_{i\in I}M(x,Y_i)
-
M(x,m(x))
\right\}
\,dP_X(x).
\end{split}
\]
Finally, for every $x$,
\[
\max_{i\in I}
|\widehat M_i(x)-M_i(x)|
\leq
\max_{i\in I}
\|\widehat M_i-M_i\|_\infty,
\]
and thus
\[
\int
\max_{i\in I}
|\widehat M_i(x)-M_i(x)|
\,dP_X(x)
\leq
\max_{i\in I}
\|\widehat M_i-M_i\|_\infty.
\]
Consequently,
\begin{equation}
\label{eqn:e17}
\begin{split}
T_2\le{}&
\max_{i\in I}\|\widehat M_i-M_i\|_\infty\\
&+
\int\left\{
\min_{i\in I}M(x,Y_i)-M(x,m(x))
\right\}\,dP_X(x).
\end{split}
\end{equation}

Proceeding as in Step 2 of the proof of Theorem \ref{theo1}, we obtain that 
\[
\begin{aligned}
&\int\left\{
\min_{i\in I}M(x,Y_i)-M(x,m(x))
\right\}\,dP_X(x)\leq
2C_n\sup_{y\in\mathcal Y}\min_{i\in I}d(Y_i,y).
\end{aligned}
\]
By condition~\eqref{cond2},
\[
C_n\sup_{y\in\mathcal Y}\min_{i\in I}d(Y_i,y)
=
o_{\mathbb P}(r_n),
\]
and hence
\[
2C_n\sup_{y\in\mathcal Y}\min_{i\in I}d(Y_i,y)
=
o_{\mathbb P}(r_n).
\]

Combining \eqref{eqn:e16}, \eqref{eqn:e18-corollary}, and \eqref{eqn:e17} gives
\[
0\le\mathcal R(\widehat m)-\mathcal R(m)
\le 2\max_{i\in I}\|\widehat M_i-M_i\|_\infty+o_{\mathbb P}(r_n)
=O_{\mathbb P}(\alpha_n+r_n).
\]
\end{proof}

\subsection{Auxiliary result}

We use the following bounded-response specialization of the entropy and norm-comparison arguments in \cite{padilla2024confidence}. The specialization is sufficient under Assumption~\ref{as1}, as shown in Step 3 above. It is stated with its own proof rather than quoted as a verbatim restatement of the general-noise theorem.

Let $\{(x_i,y_i)\}_{i=1}^n$ be independent copies of $(X,Y)$ satisfying
\begin{equation}
\label{eqn:model1}
y_i=f^*(x_i)+\epsilon_i,
\end{equation}
where $\mathbb E(\epsilon_i\mid x_i)=0$ and
$\operatorname{Var}(\epsilon_i\mid x_i)=g^*(x_i)$. Conditional on the entire design, the errors are independent. Here $f^*:[0,1]^p\to\mathbb R$ and
$g^*:[0,1]^p\to[0,\infty)$, and
\[
\|g\|_{\mathcal L_2}:=\left(\int g(x)^2\,dP_X(x)\right)^{1/2},\qquad
\|g\|_n:=\left(\frac1n\sum_{i=1}^n g(x_i)^2\right)^{1/2}.
\]

\begin{definition}
\label{def1}
Given a function class $\mathcal F\subset B_\infty(1)$, call $\delta_n>0$ a critical radius if
\[
\mathbb E\!\left[
\sup_{f\in\operatorname{star}(\mathcal F):\,\|f\|_n\le\delta_n}
\frac1n\sum_{i=1}^n\xi_i f(x_i)\,\middle|\,x_1,\ldots,x_n
\right]\le\delta_n^2,
\]
where the $\xi_i$ are independent Rademacher variables independent of the design and
$\operatorname{star}(\mathcal F)=\{\lambda f:0\le\lambda\le1,\ f\in\mathcal F\}$.
Here $B_\infty(1)$ is the unit ball for the sup norm, and the empirical norm is normalized as above.
\end{definition}

\begin{theorem}
\label{thm1_v2}
Let $\mathcal F$ be a deterministic, pointwise-separable class of measurable real functions, and let $\bar f\in\mathcal F$ satisfy
$\|\bar f-f^*\|_\infty\le\sqrt{\phi_n}$. Let
\[
\widehat f\in\argmin_{f\in\mathcal F}\frac1n\sum_{i=1}^n(y_i-f(x_i))^2
\]
be an exact empirical minimizer.
Suppose $|Y|\le\mathcal A_n/8$ almost surely and
\[
\mathcal A_n\ge8\max\{1,\|f^*\|_\infty+\mathcal U_n,
                       8\|f^*\|_\infty,8\sqrt{\phi_n}\}.
\]
Write $\widehat f_{\mathcal A_n}=\widehat f^{\mathcal A_n}$,
$\mathcal F_{\mathcal A_n}=\{f^{\mathcal A_n}/\mathcal A_n:f\in\mathcal F\}$,
and, for analysis only,
\[
\mathcal H_n=
\{(f^{\mathcal A_n}-f^*)/(2\mathcal A_n):f\in\mathcal F\}\subset B_\infty(1).
\]
Suppose a decreasing $\eta_n:(0,1)\to[0,\infty)$ bounds the empirical log covering numbers of both normalized classes, uniformly over all finite designs. The dyadic regularity conditions used in the general entropy formulation are
\begin{equation}
\label{eqn:entropy_0_v2}
\lim_{n\to\infty}\left[
\sum_{l=0}^\infty\sum_{l'=1}^\infty e^{-C_1\eta_n(2^{-l-l'-1})}
+\sum_{l=0}^\infty e^{-C_2\eta_n(2^{-l-1})}
+\mathbb P(\|\epsilon\|_\infty>\mathcal U_n)
\right]=0
\end{equation}
and
\begin{equation}
\label{eqn:cond_sum_v2}
\sup_{l\ge0}\sum_{l'=1}^\infty
\frac{\eta_n(2^{-l-l'})}{2^{2l'}\eta_n(2^{-l})}\le1,
\end{equation}
where denominators are assumed positive for positive constants $C_1$ and $C_2$. They are sufficient regularity conditions and are verified in the application; the bounded-response argument below in fact uses only the covering bound.
Choose
\[
\delta_n=c\sqrt{\frac{\log n+\eta_n(1/(24n))}{n}},
\]
with a sufficiently large numerical constant $c$. This is a critical-radius upper bound for $\mathcal H_n$ in Definition~\ref{def1}. Then, for constants $C,C'>0$ independent of $n$,
\begin{equation}
\label{eqn:claim1_v2}
\max\{\|f^*-\widehat f_{\mathcal A_n}\|_n^2,
       \|f^*-\widehat f_{\mathcal A_n}\|_{\mathcal L_2}^2\}
\le C\{\phi_n+\mathcal A_n^2\delta_n^2\}
\end{equation}
with probability at least $1-C'n^{-3}$, for all $n\ge3$. In particular, the conclusion is a high-probability upper bound, not a little-$o_{\mathbb P}$ assertion. 
\end{theorem}

\begin{proof}
The entropy-to-critical-radius bound and uniform comparison of empirical and population squared norms follow from the critical-radius calculation and the norm-comparison bound proved in Lemma 6 of \cite{padilla2024confidence}, applied to the \emph{centered} class $\mathcal H_n$. In particular, increasing $c$ if necessary, with probability at least $1-C'n^{-3}$,
\begin{equation}
    \label{eqn:norm_event}
    \|h\|_{\mathcal L_2}^2\le2\|h\|_n^2+C\delta_n^2
\quad\text{for every }h\in\mathcal H_n.
\end{equation}
This application uses the normalized empirical norm and a design-uniform covering bound. Translating by $f^*$ in this analysis does not change the class over which the estimator is fitted.

For completeness, we briefly justify the stochastic bound used below.
Conditional on the design points $x_1,\ldots,x_n$, let
$\{h_1,\ldots,h_N\}$ be a $1/(24n)$-net of $\mathcal H_n$ in the
empirical norm $\|\cdot\|_n$, chosen so that
\[
\log N\leq \eta_n\!\left(\frac{1}{24n}\right).
\]
Since $|Y|\leq\mathcal A_n/8$ almost surely and
$f^*(x)=\mathbb E(Y\mid X=x)$, we have
\[
|f^*(x)|\leq\frac{\mathcal A_n}{8},
\qquad
|\epsilon_i|
=
|Y_i-f^*(x_i)|
\leq
\frac{\mathcal A_n}{4}.
\]
Hence, conditional on $x_1,\ldots,x_n$, the variables
$\epsilon_1,\ldots,\epsilon_n$ are independent, centered, and uniformly
bounded.

For a fixed net element $h_k$, define
\[
v_k:=2\mathcal A_n h_k.
\]
Conditional on $x_1,\ldots,x_n$, the values
$v_k(x_1),\ldots,v_k(x_n)$ are deterministic, whereas
$\epsilon_1,\ldots,\epsilon_n$ are independent and centered. Moreover,
$|\epsilon_i|\leq\mathcal A_n/4$ almost surely. Hence Hoeffding's
inequality implies that, for every $t>0$,
\begin{align}
&\mathbb P\left(
\left|
\frac{2}{n}\sum_{i=1}^n\epsilon_i v_k(x_i)
\right|>t
\,\middle|\,
x_1,\ldots,x_n
\right)
\nonumber\\
&\qquad\leq
2\exp\left\{
-\frac{2nt^2}
{\mathcal A_n^2\|v_k\|_n^2}
\right\},
\label{eq:hoeffding-net-element}
\end{align}
where
\[
\|v_k\|_n^2
:=
\frac1n\sum_{i=1}^n v_k(x_i)^2.
\]
In particular, taking
\[
t=C\mathcal A_n\delta_n\|v_k\|_n
\]
for a sufficiently large numerical constant $C$ gives
\[
\mathbb P\left(
\left|
\frac{2}{n}\sum_{i=1}^n\epsilon_i v_k(x_i)
\right|
>
C\mathcal A_n\delta_n\|v_k\|_n
\,\middle|\,
x_1,\ldots,x_n
\right)
\leq
2\exp\{-c n\delta_n^2\}
\]
for some numerical constant $c>0$.

Let $N$ denote the cardinality of the
$1/(24n)$-net. By construction,
\[
\log N
\leq
\eta_n\left(\frac1{24n}\right).
\]
Therefore, a union bound over all $N$ net elements yields
\begin{align}
&\mathbb P\left(
\max_{1\leq k\leq N}
\frac{
\left|
\frac{2}{n}\sum_{i=1}^n
\epsilon_i v_k(x_i)
\right|
}{
\mathcal A_n\|v_k\|_n
}
>
C\delta_n
\,\middle|\,
x_1,\ldots,x_n
\right)
\nonumber\\
&\qquad\leq
2N\exp\{-c n\delta_n^2\}
\nonumber\\
&\qquad\leq
2\exp\left\{
\eta_n\left(\frac1{24n}\right)
-c n\delta_n^2
\right\}.
\label{eq:union-net}
\end{align}
Since
\[
\delta_n^2
=
c_0^2
\frac{
\log n+\eta_n(1/(24n))
}{n},
\]
with $c_0$ chosen sufficiently large, the last expression is bounded
by $C'n^{-3}$. Consequently, with conditional probability at least
$1-C'n^{-3}$,
\[
\left|
\frac{2}{n}\sum_{i=1}^n
\epsilon_i v_k(x_i)
\right|
\leq
C\mathcal A_n\delta_n\|v_k\|_n
\]
simultaneously for every element $v_k$ of the finite net.

Young's inequality then implies
\[
\left|
\frac{2}{n}\sum_{i=1}^n\epsilon_i v_k(x_i)
\right|
\leq
\frac18\|v_k\|_n^2
+
C\mathcal A_n^2\delta_n^2.
\]

Now let $v=f^{\mathcal A_n}-f^*$ be arbitrary, and write $v=2\mathcal A_n h$ with
$h\in\mathcal H_n$. Let $h_k$ be a net element satisfying
\[
\|h-h_k\|_n\leq\frac{1}{24n},
\]
and define
\[
v_k:=2\mathcal A_n h_k.
\]
Then
\begin{equation}
\label{eq:net-v-distance}
\|v-v_k\|_n
=
2\mathcal A_n\|h-h_k\|_n
\leq
\frac{\mathcal A_n}{12n}.
\end{equation}

We first control the change in the multiplier term. By the
Cauchy--Schwarz inequality,
\begin{align}
\left|
\frac{2}{n}
\sum_{i=1}^n
\epsilon_i\{v(x_i)-v_k(x_i)\}
\right|
&\leq
2\|\epsilon\|_n\|v-v_k\|_n.
\label{eq:multiplier-net-difference}
\end{align}
Since $|\epsilon_i|\leq\mathcal A_n/4$ for every $i$,
\[
\|\epsilon\|_n
=
\left(
\frac1n\sum_{i=1}^n\epsilon_i^2
\right)^{1/2}
\leq
\frac{\mathcal A_n}{4}.
\]
Combining this inequality with \eqref{eq:net-v-distance} gives
\begin{equation}
\label{eq:multiplier-net-remainder}
\left|
\frac{2}{n}
\sum_{i=1}^n
\epsilon_i\{v(x_i)-v_k(x_i)\}
\right|
\leq
\frac{\mathcal A_n^2}{24n}.
\end{equation}

Next, by the triangle inequality,
\[
\|v_k\|_n
\leq
\|v\|_n+\|v-v_k\|_n.
\]
Using $(a+b)^2\leq2a^2+2b^2$ and
\eqref{eq:net-v-distance},
\begin{align}
\frac18\|v_k\|_n^2
&\leq
\frac14\|v\|_n^2
+
\frac14\|v-v_k\|_n^2
\nonumber\\
&\leq
\frac14\|v\|_n^2
+
\frac{\mathcal A_n^2}{576n^2}.
\label{eq:quadratic-net-remainder}
\end{align}

On the high-probability event established for the finite net,
\[
\left|
\frac{2}{n}
\sum_{i=1}^n
\epsilon_i v_k(x_i)
\right|
\leq
\frac18\|v_k\|_n^2
+
C\mathcal A_n^2\delta_n^2.
\]
Therefore, by adding and subtracting $v_k$ and using
\eqref{eq:multiplier-net-remainder} and
\eqref{eq:quadratic-net-remainder},
\begin{align}
\left|
\frac{2}{n}
\sum_{i=1}^n
\epsilon_i v(x_i)
\right|
&\leq
\left|
\frac{2}{n}
\sum_{i=1}^n
\epsilon_i v_k(x_i)
\right|
+
\left|
\frac{2}{n}
\sum_{i=1}^n
\epsilon_i\{v(x_i)-v_k(x_i)\}
\right|
\nonumber\\
&\leq
\frac14\|v\|_n^2
+
C\mathcal A_n^2\delta_n^2
+
\frac{\mathcal A_n^2}{24n}
+
\frac{\mathcal A_n^2}{576n^2}.
\label{eq:extend-net-preabsorb}
\end{align}
Since
\[
\delta_n^2
\gtrsim
\frac{\log n}{n},
\]
the last two terms in \eqref{eq:extend-net-preabsorb} are bounded by
$C\mathcal A_n^2\delta_n^2$ for $n$ sufficiently large. Hence, after
enlarging the constant $C$,
\begin{equation}
\label{eq:multiplier-full-class}
\left|
\frac{2}{n}
\sum_{i=1}^n
\epsilon_i v(x_i)
\right|
\leq
\frac14\|v\|_n^2
+
C\mathcal A_n^2\delta_n^2,
\qquad
v=f^{\mathcal A_n}-f^*,\quad f\in\mathcal F,
\end{equation}
simultaneously over the entire class.

Because $|Y_i|\leq \mathcal A_n/8<\mathcal A_n$, every observed response
$Y_i$ lies in the clipping interval $[-\mathcal A_n,\mathcal A_n]$.
For any $z\in\mathbb R$, projection of $z$ onto this interval cannot
increase its distance to $Y_i$. Hence,
\[
\{Y_i-\widehat f^{\mathcal A_n}(x_i)\}^2
\leq
\{Y_i-\widehat f(x_i)\}^2
\qquad\text{for every }i.
\]
Moreover,
\[
\|\bar f\|_\infty
\leq
\|f^*\|_\infty+\|\bar f-f^*\|_\infty
\leq
\|f^*\|_\infty+\sqrt{\phi_n}
<\mathcal A_n,
\]
by the assumed lower bound on $\mathcal A_n$. Therefore
$\bar f^{\mathcal A_n}=\bar f$.

Since $\widehat f$ is an exact empirical least-squares minimizer and
$\bar f\in\mathcal F$, it follows that
\[
\frac1n\sum_{i=1}^n
\{Y_i-\widehat f_{\mathcal A_n}(x_i)\}^2
\leq
\frac1n\sum_{i=1}^n
\{Y_i-\bar f(x_i)\}^2.
\]
Using $Y_i=f^*(x_i)+\epsilon_i$ and expanding both squared terms gives
\begin{align}
\|\widehat f_{\mathcal A_n}-f^*\|_n^2
&\leq
\|\bar f-f^*\|_n^2
\nonumber\\
&\quad+
\frac2n\sum_{i=1}^n\epsilon_i
\left[
\{\widehat f_{\mathcal A_n}-f^*\}(x_i)
-
\{\bar f-f^*\}(x_i)
\right].
\label{eq:basic-erm-expanded}
\end{align}

On the high-probability multiplier event established above, the bound
\[
\left|
\frac2n\sum_{i=1}^n\epsilon_i v(x_i)
\right|
\leq
\frac14\|v\|_n^2
+
C\mathcal A_n^2\delta_n^2
\]
holds simultaneously for every
$v=f^{\mathcal A_n}-f^*$ with $f\in\mathcal F$.
Applying it first to
$v=\widehat f_{\mathcal A_n}-f^*$ and then to
$v=\bar f-f^*$, and using
\[
\|\bar f-f^*\|_n^2
\leq
\|\bar f-f^*\|_\infty^2
\leq
\phi_n,
\]
we obtain from \eqref{eq:basic-erm-expanded}
\[
\|\widehat f_{\mathcal A_n}-f^*\|_n^2
\leq
\frac54\phi_n
+
\frac14
\|\widehat f_{\mathcal A_n}-f^*\|_n^2
+
C\mathcal A_n^2\delta_n^2.
\]
Therefore,
\[
\frac34
\|\widehat f_{\mathcal A_n}-f^*\|_n^2
\leq
\frac54\phi_n
+
C\mathcal A_n^2\delta_n^2,
\]
and hence, after changing the numerical constant,
\begin{equation}
\label{eq:empirical-estimation-bound}
\|\widehat f_{\mathcal A_n}-f^*\|_n^2
\leq
C\left\{
\phi_n+\mathcal A_n^2\delta_n^2
\right\}.
\end{equation}

Recall that
\[
h
=
\frac{\widehat f_{\mathcal A_n}-f^*}{2\mathcal A_n}.
\]
By homogeneity of the population and empirical $L^2$ norms,
\[
\|h\|_{\mathcal L_2}^2
=
\frac{1}{4\mathcal A_n^2}
\|\widehat f_{\mathcal A_n}-f^*\|_{\mathcal L_2}^2,
\qquad
\|h\|_n^2
=
\frac{1}{4\mathcal A_n^2}
\|\widehat f_{\mathcal A_n}-f^*\|_n^2.
\]
Therefore, on the population--empirical norm comparison event,
\[
\frac{1}{4\mathcal A_n^2}
\|\widehat f_{\mathcal A_n}-f^*\|_{\mathcal L_2}^2
\leq
\frac{2}{4\mathcal A_n^2}
\|\widehat f_{\mathcal A_n}-f^*\|_n^2
+
C\delta_n^2.
\]
Multiplying both sides by $4\mathcal A_n^2$ yields
\[
\|\widehat f_{\mathcal A_n}-f^*\|_{\mathcal L_2}^2
\leq
2\|\widehat f_{\mathcal A_n}-f^*\|_n^2
+
4C\mathcal A_n^2\delta_n^2.
\]
Absorbing the numerical factor $4$ into the generic constant $C$ gives
\[
\|\widehat f_{\mathcal A_n}-f^*\|_{\mathcal L_2}^2
\leq
2\|\widehat f_{\mathcal A_n}-f^*\|_n^2
+
C\mathcal A_n^2\delta_n^2.
\]
Combining this inequality with
\eqref{eq:empirical-estimation-bound} gives
\[
\max\left\{
\|\widehat f_{\mathcal A_n}-f^*\|_n^2,
\|\widehat f_{\mathcal A_n}-f^*\|_{\mathcal L_2}^2
\right\}
\leq
C\left\{
\phi_n+\mathcal A_n^2\delta_n^2
\right\}.
\]
Finally, a union bound over the multiplier event and the
population--empirical norm comparison event gives the asserted
probability bound.
\end{proof}

\subsection{Proof of Theorem \ref{theo2}}

\begin{proof}

We verify the assumptions of Theorem~\ref{theo1} uniformly in $y$.
The uniform bounds on the hierarchy depth and on $K$ imply a uniform
bound on the number of component functions in the chosen hierarchical
representations. Together with the uniform smoothness, derivative, and
Lipschitz bounds, this ensures that the constants in Theorem 3(a)--(b)
of \cite{kohler2021rate} can be chosen independently of $y$.

Fix $y\in\mathcal Y$. By Theorem 3(a)--(b) of
\cite{kohler2021rate}, applied to
$M(\cdot,y)\in\mathcal H(l(y),\mathcal P(y))$ on $[0,1]^p$, one can
construct a fully connected ReLU network
\[
f_y^{\mathrm{app}}\in\mathcal F(L(y),\nu(y))
\]
whose approximation error is controlled by
\[
\|f_y^{\mathrm{app}}-M(\cdot,y)\|_\infty
\lesssim
\max_{(t,K)\in\mathcal P(y)}M_{t,K}^{-2t},
\]
where $M_{t,K}$ is the approximation parameter used in the construction
for a component with smoothness $t$ and input dimension $K$.

Choose
\[
M_{t,K}\asymp n^{1/[2(2t+K)]}.
\]
Then
\[
M_{t,K}^{-2t}
\asymp
n^{-t/(2t+K)}
\]
and
\[
M_{t,K}^{K}
\asymp
n^{K/[2(2t+K)]}.
\]
The quantity $M_{t,K}^{K}$ determines the polynomial growth in the
network architecture supplied by Theorem 3(a)--(b) of
\cite{kohler2021rate}. In the construction corresponding to
\eqref{eqn:v2_choice1}, this growth is accommodated by the width
$\nu(y)$ while $L(y)$ is of logarithmic order. In the construction
corresponding to \eqref{eqn:v2_choice2}, the width remains bounded and
the same polynomial growth is accommodated by the depth $L(y)$,
together with the additional factor $\log n$.

Because all constants in the approximation theorem are uniform in
$y$, the constants $c_3,c_4$ in
\eqref{eqn:v2_choice1}--\eqref{eqn:v2_choice2} can be chosen
sufficiently large, independently of $y$, so that
$f_y^{\mathrm{app}}\in\mathcal F(L(y),\nu(y))$ and
\[
\|f_y^{\mathrm{app}}-M(\cdot,y)\|_\infty
\leq
\max_{(t,K)\in\mathcal P(y)}
n^{-t/(2t+K)}.
\]

By definition,
\[
\widetilde M(\cdot,y)
\in
\argmin_{f\in\mathcal F(L(y),\nu(y))}
\|f-M(\cdot,y)\|_\infty.
\]
Since $f_y^{\mathrm{app}}\in\mathcal F(L(y),\nu(y))$, it follows that
\[
\begin{aligned}
\|\widetilde M(\cdot,y)-M(\cdot,y)\|_\infty
&\leq
\|f_y^{\mathrm{app}}-M(\cdot,y)\|_\infty\\
&\leq
\max_{(t,K)\in\mathcal P(y)}
n^{-t/(2t+K)}.
\end{aligned}
\]
Taking the supremum over $y\in\mathcal Y$ therefore gives
\begin{equation}
\label{eqn:60}
\begin{split}
\sup_{y\in\mathcal Y}
\|\widetilde M(\cdot,y)-M(\cdot,y)\|_\infty
&\leq
\sup_{y\in\mathcal Y}
\max_{(t,K)\in\mathcal P(y)}
n^{-t/(2t+K)}\\
&=
\sqrt{
\sup_{y\in\mathcal Y}
\max_{(t,K)\in\mathcal P(y)}
n^{-2t/(2t+K)}
}\\
&=
\sqrt{\phi_n}.
\end{split}
\end{equation}

We next bound the network-complexity term in Theorem~\ref{theo1}.
Under \eqref{eqn:v2_choice1},
$L_i\asymp\log n$ and
\[
\nu_i
\lesssim
\max_{(t,K)\in\mathcal P(Y_i)}
n^{K/[2(2t+K)]},
\]
whereas under \eqref{eqn:v2_choice2},
$\nu_i\asymp1$ and
\[
L_i
\lesssim
\max_{(t,K)\in\mathcal P(Y_i)}
n^{K/[2(2t+K)]}\log n.
\]
In either case,
\[
\log(L_i\nu_i)\lesssim\log n,
\]
and hence
\begin{equation}
\label{eqn:61}
\begin{split}
\max_{i\in I}\{L_i^2\nu_i^2\log(L_i\nu_i)\}
&\lesssim
\max_{i\in I}\max_{(t,K)\in\mathcal P(Y_i)}
n^{K/(2t+K)}(\log n)^3\\
&\leq
\sup_{y\in\mathcal Y}\max_{(t,K)\in\mathcal P(y)}
n^{K/(2t+K)}(\log n)^3\\
&=
n\phi_n(\log n)^3,
\end{split}
\end{equation}
where the final equality follows from
\[
n^{K/(2t+K)}
=
n\,n^{-2t/(2t+K)}.
\]
The ceiling operations affect only multiplicative constants.

 Let
\[
\overline t
:=
\sup_{y\in\mathcal Y}
\max_{(t,K)\in\mathcal P(y)}t
<\infty.
\]
Since $K\geq1$,
\[
\frac{K}{2(2t+K)}
\geq
\frac{1}{2(2\overline t+1)}
=:c>0.
\]
Therefore,
\[
\max_{(t,K)\in\mathcal P(y)}
n^{K/[2(2t+K)]}
\geq n^c
\]
uniformly in $y$. Under \eqref{eqn:v2_choice1}, this implies
$\nu_i\gtrsim n^c$ while $L_i\asymp\log n$; under
\eqref{eqn:v2_choice2}, it implies
$L_i\gtrsim n^c\log n$ while $\nu_i\asymp1$. Hence, under either
architecture,
\[
\min_{i\in I}
\{L_i^2\nu_i^2\log(L_i\nu_i)\}
\gtrsim
n^{2c}\log n.
\]
Since $A_n\geq8$, $\log(2A_n^2)$ is bounded below by a positive
constant, and therefore
\[
\frac{
\min_{i\in I}
\{L_i^2\nu_i^2\log(L_i\nu_i)\}
\log(2A_n^2)
}{
\log n
}
\longrightarrow\infty.
\]
Thus \eqref{eqn:cond1} holds for either architecture.

Finally, substituting \eqref{eqn:60} and \eqref{eqn:61} into
Theorem~\ref{theo1} yields
\[
\begin{split}
\mathcal R(\widehat m)-\mathcal R(m)
\leq r_n+C\Big[&
\sqrt{
|I|\max\{\mathcal U_n^2,A_n^2\}
\phi_n(\log n)^3\log(A_nn)
}\\
&+\sqrt{|I|\phi_n}
\Big]
\end{split}
\]
with probability approaching one. Since $A_n\geq8$ and $n\geq3$, the
second term in brackets is absorbed by the first. This proves
Theorem~\ref{theo2} and its expanded formulation,
Theorem~\ref{v2:theo2}.

\end{proof}

\printbibliography
\end{document}